%% file: main.tex
\documentclass[sigconf,screen,authorversion,nonacm]{acmart}

\AtBeginDocument{%
  \providecommand\BibTeX{{%
    \normalfont B\kern-0.5em{\scshape i\kern-0.25em b}\kern-0.8em\TeX}}}

\copyrightyear{2022}
\acmYear{2022}
\setcopyright{acmcopyright}\acmConference[ICSE '22]{44th International Conference on Software Engineering}{May 21--29, 2022}{Pittsburgh, PA, USA}
\acmBooktitle{44th International Conference on Software Engineering (ICSE '22), May 21--29, 2022, Pittsburgh, PA, USA}
\acmPrice{15.00}
\acmDOI{10.1145/3510003.3510143}
\acmISBN{978-1-4503-9221-1/22/05}

\acmConference[ICSE 2022]{The 44th International Conference on Software Engineering}{May 21–29, 2022}{Pittsburgh, PA, USA}

\usepackage{booktabs}
\usepackage{marginnote}
\usepackage{multirow} 
\usepackage{bm}
\usepackage[caption=false]{subfig}
\usepackage{tikz}
\usepackage{color}
\usepackage{fancybox,framed}
\usetikzlibrary{arrows,automata}

\def\layersep{2.5cm}

\newcommand{\R}{\mathbb{R}}

\newcommand{\x}{\bm{x}}
\renewcommand{\c}{\bm{c}}

\usepackage{xspace}
\newcommand{\deeppac}{{\sc DeepPAC}\xspace}
\newcommand{\eran}{{\sc ERAN}\xspace}
\newcommand{\provero}{{\sc PROVERO}\xspace}
\newcommand{\deepgini}{{\sc DeepGini}\xspace}
\newcommand{\commentout}[1]{}

\usepackage{color, colortbl}
\definecolor{LightCyan}{rgb}{0.88,1,1}
\definecolor{Gray}{gray}{0.9}

\begin{document}

\title{Towards Practical Robustness Analysis for DNNs based on PAC-Model Learning}

\author{Renjue Li}
\email{lirj19@ios.ac.cn}
\affiliation{%
  \institution{SKLCS, Institute of Software, CAS \and
	University of Chinese Academy of Sciences}
  \country{China}
}

\author{Pengfei Yang}
\email{yangpf@ios.ac.cn}
\authornote{Corresponding authors}
\affiliation{%
  \institution{SKLCS, Institute of Software, CAS}
  \country{China}
}

\author{Cheng-Chao Huang}
\email{chengchao@nj.iscas.ac.cn}
\affiliation{%
  \institution{Nanjing Institute of Software Technology, ISCAS \and
  Pazhou Lab}
  \country{China}
}

\author{Youcheng Sun}
\email{youcheng.sun@qub.ac.uk}
\affiliation{%
  \institution{Queen's University Belfast}
  \country{United Kingdom}
}

\author{Bai Xue}
\email{xuebai@ios.ac.cn}
\affiliation{%
  \institution{SKLCS, Institute of Software, CAS \and
	University of Chinese Academy of Sciences}
  \country{China}
}

\author{Lijun Zhang}
\email{zhanglj@ios.ac.cn}
\authornotemark[1]
\affiliation{%
  \institution{SKLCS, Institute of Software, CAS \and
	University of Chinese Academy of Sciences}
  \country{China}
}

\renewcommand{\shortauthors}{Renjue Li, et al.}

\begin{abstract}

To analyse local robustness properties of deep neural networks (DNNs), we present a practical framework from a model learning perspective. Based on black-box model learning with scenario optimisation, we abstract the local behaviour of a DNN via an affine model with the probably approximately correct (PAC) guarantee. From the learned model, we can infer the corresponding PAC-model robustness property. The innovation of our work is the integration of model learning into PAC robustness analysis: that is, we construct a PAC guarantee on the model level instead of sample distribution, which induces a more faithful and accurate robustness evaluation. This is in contrast to existing statistical methods without model learning. 
We implement our method in a prototypical tool named \deeppac. As a black-box method, \deeppac is scalable and efficient, especially when DNNs have complex structures or high-dimensional inputs. We extensively evaluate \deeppac, with 4 baselines (using formal verification, statistical methods, testing and adversarial attack) and 20 DNN models across 3 datasets, including MNIST, CIFAR-10, and ImageNet. It is shown that \deeppac outperforms the state-of-the-art statistical method \provero, and it achieves more practical robustness analysis than the formal verification tool \eran. Also, its results are consistent with existing DNN testing work like \deepgini.

\end{abstract}



\maketitle

\input{introduction}

\input{preliminary}

\input{pac-model-robustness}

\input{method}

\section{Experimental EVALUATION} \label{sec:experiment}
In this section, we evaluate our PAC-model robustness verification method.
 We implement our algorithm as a prototype called \deeppac.  Its implementation is based on Python 3.7.8. We use CVXPY~\cite{diamond2016cvxpy} as the modeling language for linear programming and GUROBI~\cite{gurobi} as the LP solver. Experiments are conducted on a Windows 10  PC with Intel i7 8700, GTX 1660Ti, and 16G RAM.
Three datasets MNIST~\cite{L1998Gradient}, CIFAR-10~\cite{cifar10}, and ImageNet~\cite{ILSVRC15} and 20 DNN models trained from them are used in the evaluation. The details are in Tab.~\ref{tab:networks}. We invoke our component-based learning and focused learning for all evaluations, and apply stepwise splitting for the experiment on ImageNet. All the implementation and data used in this section are publicly available\footnote{\url{https://github.com/CAS-LRJ/DeepPAC}}.


In the following, we are going to answer the research questions below.
\begin{itemize}
    \item[{\bf RQ1:}] Can \deeppac evaluate local robustness of a DNN more effectively compared with the state-of-the-art?
    \item[{\bf RQ2:}] Can \deeppac retain a reasonable accuracy with higher significance, higher error rate, and/or fewer samples?
    \item[{\bf RQ3:}] Is \deeppac scalable to DNNs with complex structure and high dimensional input?
    \item[{\bf RQ4:}] Is there a underlying relation between DNN local robustness verification and DNN testing (especially the test selection)? 
\end{itemize}

\commentout{
\doublebox{
\begin{minipage}{0.94\linewidth}
\begin{enumerate}
    \item[{\bf RQ1:}] Can \deeppac evaluate local robustness of a DNN more effectively comparing with the state-of-the-art?
    \item[{\bf RQ2:}] Can \deeppac retain a reasonable accuracy with higher significance, higher error rate, or fewer samples?
    \item[{\bf RQ3:}] Is \deeppac scalable to DNNs with complex structures and high dimensional inputs?
    \item[{\bf RQ4:}] Is there a underlying consistency between the maximum robustness radius and the testing prioritising metric?
\end{enumerate}
\end{minipage}
}
}


\begin{table}[t]
\begin{tabular}{|c|l|l|r|c|}
\hline
Dataset & Network & Defense & \#Param\, & Source \\ \hline
\multirow{12}{*}{MNIST} & FNN1 & \multirow{6}{*}{\qquad  \;\, ---} & 44.86\,K & \multirow{6}{*}{---} \\ \cline{2-2} \cline{4-4}
 & FNN2 &  & 99.71\,K &  \\ \cline{2-2} \cline{4-4}
 & FNN3 &  & 239.41\,K &  \\ \cline{2-2} \cline{4-4}
 & FNN4 &  & 360.01\,K &  \\ \cline{2-2} \cline{4-4}
 & FNN5 &  & 480.61\,K &  \\ \cline{2-2} \cline{4-4}
 & FNN6 &  & 1.65\,M &  \\ \cline{2-5} 
 & CNN1 & \qquad \;\, --- & \multirow{3}{*}{89.61\,K} & \multirow{9}{*}{ERAN} \\ \cline{2-3}
 & CNN2 &  DiffAI &  &  \\ \cline{2-3}
 & CNN3 & PGD &  &  \\ \cline{2-4}
 & CNN4 & \qquad \;\, --- & \multirow{3}{*}{1.59\,M} &  \\ \cline{2-3}
 & CNN5 & PGD, $\varepsilon = 0.1$ &  &  \\ \cline{2-3}
 & CNN6 & PGD, $\varepsilon = 0.3$ &  &  \\ \cline{1-4}
\multirow{6}{*}{CIFAR-10} & CNN1 & PGD & 125.32\,K &  \\ \cline{2-4}
 & CNN2 & PGD, $\varepsilon =2/255$ & \multirow{2}{*}{2.07\,M} &  \\ \cline{2-3}
 & CNN3 & \multirow{5}{*}{PGD, $\varepsilon =8/255$} &  &  \\ \cline{2-2} \cline{4-5} 
 & ResNet18 &  & 11.17\,M & \multirow{3}{*}{---} \\ \cline{2-2} \cline{4-4}
 & ResNet50 &  & 23.52\,M &  \\ \cline{2-2} \cline{4-4}
 & ResNet152 &  & 58.16\,M &  \\ \hline
\multirow{2}{*}{ImageNet} & ResNet50a & PGD, $\varepsilon =4/255$ & \multirow{2}{*}{25.56\,M} & \multirow{2}{*}{Madry} \\ \cline{2-3}
 & ResNet50b & PGD, $\varepsilon =8/255$ &  &  \\ \hline
\end{tabular}
\caption{Datasets and DNNs used in our evaluation. The convolutional neural networks (CNN) for MNIST and CIFAR-10 are from ERAN~\cite{ERAN}. The ResNet50 networks for ImageNet are from the python library ``Robustness'' \cite{robustness} produced by MadryLab. The rest networks are trained by ourselves.}
\label{tab:networks}
\end{table}

\subsection{Comparison on Precision} 
\label{subsect:exp1}
\begin{figure*}[t]
    \centering
        \centering
        \includegraphics[width=0.98\linewidth]{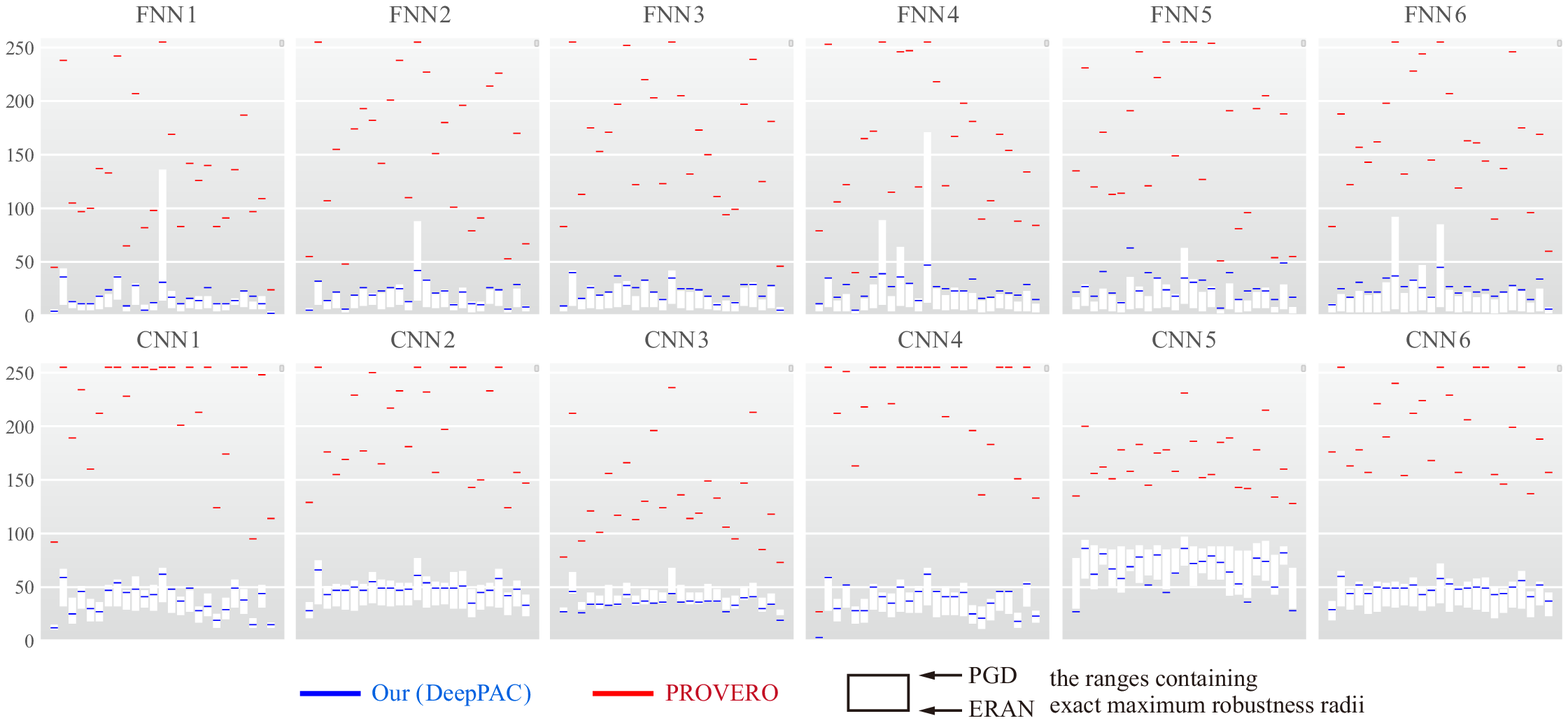}
    \caption{Each dash represents the maximum robustness radius for an input estimated by \deeppac (blue) or \provero (red), while each bar (white) gives an interval containing the exact maximum robustness radius, whose lower bound and upper bound are computed by ERAN and PGD, respectively.}
    \label{fig:netcountermnist_comb}
\end{figure*}
\commentout{
The exact maximum robustness radius is an ultimate metric to evaluate the local robustness of a DNN, which indicates the distance between the given input and the boundary of the adversarial examples. Although SMT-based tools like Marabou~\cite{marabou} are sound and complete for the local robustness verification and can theoretically compute the exact maximum robustness radius, it faces the curse of dimensionality and its running time is not acceptable in our evaluation. The exact maximum robustness radius can be estimated by a binary search over the verifiable radii for verification methods or attack radii for attack methods. }

We first apply \deeppac for evaluating DNN local robustness by computing the maximum robustness radius and compare \deeppac with the state-of-the-art statistical verification tool \provero~\cite{baluta2021scalable}, which verifies PAC robustness by statistical hypothesis testing. 
A DNN verification tool returns true or false for robustness of a DNN given a specified radius value. A binary search  will be conducted for finding the maximum robustness radius. For both \deeppac and \provero, we set the error rate $\epsilon = 0.01$ and the significance level $\eta = 0.001$. We set $K^{(1)}=2000$ and $K^{(2)}=8000$ for \deeppac.

\commentout{
The maximum robustness radius is estimated via the binary search and 
Using \deeppac, we compute the maximum robustness radius 
by verifying PAC-model robustness under the setting
of the 
The set up for error rate $\epsilon = 0.01$ and the significance level $\eta = 0.001$.
In the same setting, we also compute the maximum robustness radii by verifying PAC robustness by \provero.
}

In addition, we apply \eran \cite{deeppoly} and PGD \cite{pgd} to bound the exact maximum radius from below and from above, respectively.
\eran is a state-of-the-art DNN formal verification tool based on abstract interpretation, and PGD is a popular adversarial attack algorithm.
In the experiments, we use the PGD implementation from the commonly used Foolbox~\cite{foolbox} with $40$ iterations and a relative step size of $0.033$, which are suggested by Foolbox as a default setting.
Note that exact robustness verification SMT tools  like Marabou~\cite{marabou} cannot scale to the benchmarks used in our experiment.
\commentout{
Note that for verifying strict robustness of DNNs, 
ERAN is sound, and PGD is considered to be complete.
Thus, the maximum robustness radius computed 
by them is always a lower bound and an upper bound of the exact robustness radius, respectively.
Namely, they provide a range containing the exact robustness radius.}




We  run all the tools on the first 12 DNN models in Tab.~\ref{tab:networks} and  the detailed results are recorded in Fig.~\ref{fig:netcountermnist_comb}. 
In all cases, the maximum robustness radius estimated by the  \provero is far larger than those computed by other tools.  
In most cases, \provero ends up with a maximum robustness radius over $100$ (out of 255), which is even larger than the upper bound identified by PGD. 
This indicates that, while a DNN is proved to be PAC robust by \provero, adversarial inputs can be still rather easily found within the verified bound. In contrast, \deeppac estimates the maximum robustness radius more accurately, which falls in between the results from \eran and PGD mostly. Since the range between the estimation of \eran and PGD contains the exact maximum robustness radius, we conclude that \deeppac is a more accurate tool than \provero to analyse local robustness of DNNs.

\deeppac also successfully distinguishes  robust DNN models from non-robust ones.
It tells that the CNNs, especially the ones with defence mechanisms, are more robust against adversarial perturbations. For instance, 24 out of 25 images have a larger maximum robustness radius on CNN1 than on FNN1, and 21 images have a larger maximum robustness radius on CNN2 than on CNN1.

Other than the maximum robustness radius for a fixed input, the overall robustness of a DNN, subject to some radius value, 
 can be denoted by the rate of the inputs being robust
in a dataset, called ``robustness rate''. 
In Fig.~\ref{fig:robrate}, we show the robustness rate of 100 input images estimated by different tools on the 3 CIFAR-10 CNNs. Here, we set $K^{(1)}=20\,000$ and $K^{(2)}=10\,000$.

\begin{figure}[ht]
    \centering
        \centering
        \includegraphics[width=0.98\linewidth]{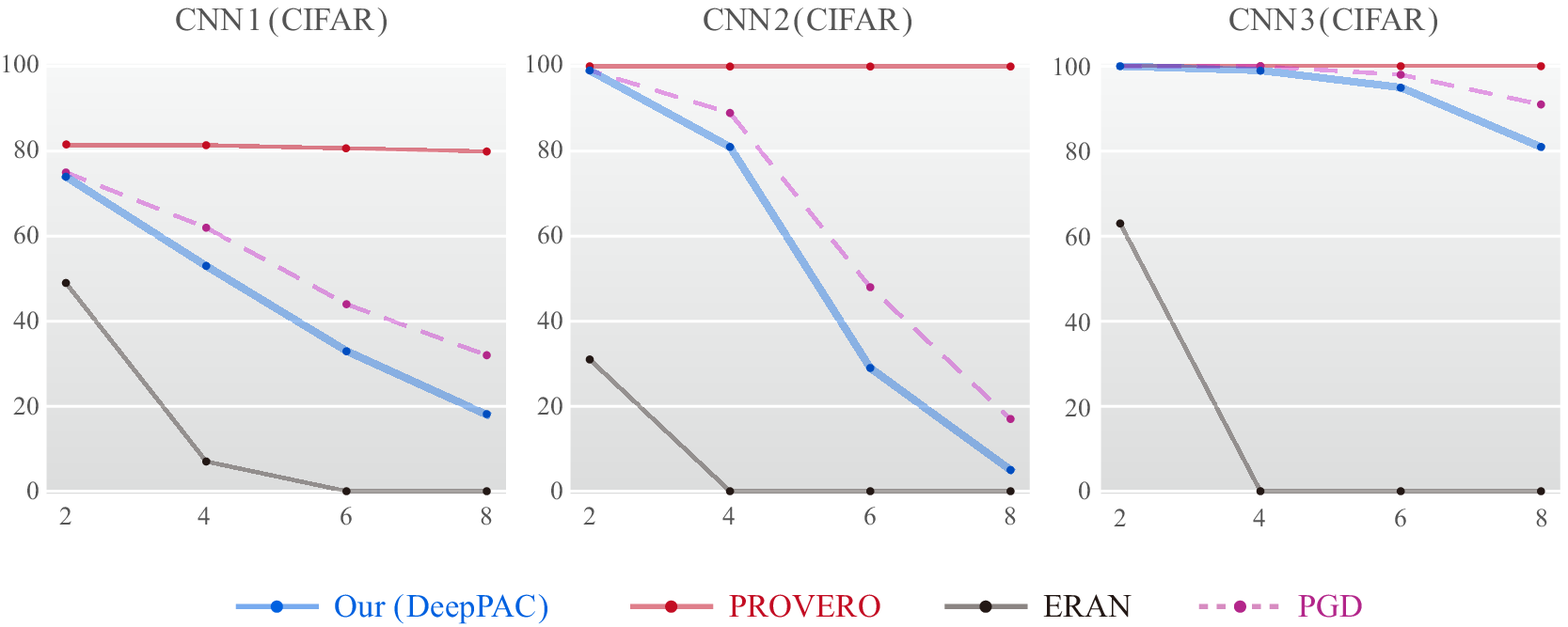}
    \caption{Robustness rate of different CNNs under the radius of 2, 4, 6, and 8 on CIFAR-10.}
    \label{fig:robrate}
\end{figure}


\provero, similarly to the earlier experiment outcome, results in robustness rate which is even higher 
than the upper bound estimation from the PGD attack, and its robustness rate result hardly changes when the robustness radius increases. All such comparisons reveal the limitations of using PAC robustness (by \provero) that the verified results are not tight enough.

\eran is a sound verification method, and the robustness rate verified by it is a strict lower bound of the exact result. However, this lower bound could be too conservative and  \eran quickly becomes not usable. In the experiments, we find that it is hard for \eran to verify a robustness radius greater than or equal to 4 (out of 255).


\deeppac verifies greater robustness rate and larger robustness radius, with high confidence and low error rate. Its results fall safely into the range bounded by \eran and PGD. We advocate \deeppac as a more practical DNN robustness analysis technique. It is shown in our experiments that, though \deeppac does not enforce 100\% guarantee, it can be applied into a wider range of adversarial settings (in contrast to \eran) and the PAC-model verification results  by \deeppac can be more trusted (in contrast to \provero) with quantified confidence (in contrast to PGD). 

\commentout{
As we expect, the robustness rate obtained by \deeppac 
is slightly lower than the failure rate of the PGD attack,
and shows a reasonable difference under various perturbation radii.
From the discussion above, we demonstrate that \deeppac is a more effective method that can be used for robustness analysis, by which we can obatin more meaningful and practical robustness rate to evaluate the overall local robustness of a DNN.
}


\noindent\doublebox{
\begin{minipage}{0.94\linewidth}
 {\bf Answer RQ1:} The maximum robustness radius estimated by \deeppac is more precise than that by \provero, and our \deeppac is a more practical DNN robustness analysis method. 
\end{minipage}
}

\subsection{\deeppac with Different Parameters}



In this part, we experiment on the three key parameters in \deeppac: the error rate $\epsilon$, the significance level $\eta$, and the number of samples $K^{(1)}$ in the first learning phase. The parameters $\eta$ and $\epsilon$ control the precision between the PAC model and the original model. The number of samples $K^{(1)}$ determines the accuracy of the first learning phase.
We evaluate \deeppac under different parameters to check the variation of the maximal robustness radius. We set either $K^{(1)}=20 000$ or $K^{(1)}=5 000$  in our evaluation and three combinations of the parameters $(\epsilon,\eta)$:  $(0.01,0.001)$, $(0.1, 0.001)$, and $(0.01,0.1)$. Here, we fix the number of key features to be fifty, i.e. $\kappa=50$, and calculate the corresponding number of samples $K^{(2)}$ in the focused learning phase. 

The results are presented in Tab.~\ref{tab:stable}.
\deeppac reveals some DNN robustness insights that were not achievable by other verification work. It is shown that, the DNNs (the ResNet family experimented) can be more robust than many may think.
The maximum robustness radius remains the same or slightly alters, along with the error rate $\eta$ and significance level $\epsilon$ varying. 
This observation also confirms that the affine model used in \deeppac abstraction converges well, and the resulting error bound is even smaller than the specified (large) error bound. Please refer to Sect.~\ref{subsec:learn_original} for more details.


\deeppac is also tolerant enough with a small sampling size. When the number of samples in the first learning phase decreases from $K^{(1)}=20,000$ to $K^{(1)}=5,000$, we can observe a minor decrease of the maximal robustness radius estimation. 
Recall that we utilise the learned model in the first phase of focused learning to extract the key features and provide coefficients to the less important features. When the sampling number decreases, the learned model would be less precise and thus make vague predictions on key features and make the resulting affine model shift from the original model. As a result, the maximum robustness radius can be smaller when we reduce the number of sampling in the first phase. In practice, as it is shown by the results in Tab.~\ref{tab:stable}, we do not observe a sudden drop of the \deeppac results when using a much smaller sampling size.

\begin{table}[ht]
\renewcommand\arraystretch{1.5}
\begin{tabular}{llcc|cc|cc}
\toprule
\multirow{3}{*}{Input Image} & \multirow{3}{*}{Network} & \multicolumn{6}{c}{$\eta,\epsilon$ and $K^{(1)}$} \\ \cline{3-8}
  &  & \multicolumn{2}{c|}{$0.01,0.001$} & \multicolumn{2}{c|}{$0.1, 0.001$} & \multicolumn{2}{c}{$0.01,0.1$} \\ \cline{3-8}
 &  & 20K & 5K & 20K & 5K & 20K & 5K \\ \hline
\multirow{3}{*}{{\begin{minipage}{0.08\textwidth}
\centering
 \includegraphics[width=1\linewidth]{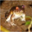}
\end{minipage}}} & \cellcolor{Gray} ResNet18 & \cellcolor{Gray} 5 & \cellcolor{Gray} 4 & \cellcolor{Gray} 5 & \cellcolor{Gray} 4 & \cellcolor{Gray} 5 & \cellcolor{Gray} 4 \\ 
 & \cellcolor{Gray} ResNet50 & \cellcolor{Gray} 8 & \cellcolor{Gray} 8 & \cellcolor{Gray} 8 & \cellcolor{Gray} 8 & \cellcolor{Gray} 9 & \cellcolor{Gray} 8 \\ 
 & \cellcolor{Gray} ResNet152 & \cellcolor{Gray} 5 & \cellcolor{Gray} 5 & \cellcolor{Gray} 5 & \cellcolor{Gray} 5 & \cellcolor{Gray} 5 & \cellcolor{Gray} 5 \\ 
\multirow{3}{*}{{\begin{minipage}{0.08\textwidth}
\centering
\includegraphics[width=1\linewidth]{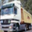}
\end{minipage}}} & ResNet18 & 16 & 14 & 15 & 14 & 15 & 14 \\ 
 & ResNet50 & 12 & 11 & 12 & 12 & 12 & 11 \\ 
 & ResNet152 & 10 & 9 & 10 & 9 & 10 & 9 \\ 
\multirow{3}{*}{{\begin{minipage}{0.08\textwidth}
\centering
\includegraphics[width=1\linewidth]{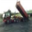}
\end{minipage}}} & \cellcolor{Gray}ResNet18 & \cellcolor{Gray}11 & \cellcolor{Gray}10 & \cellcolor{Gray}11 & \cellcolor{Gray}10 & \cellcolor{Gray}11 & \cellcolor{Gray}10 \\ 
 & \cellcolor{Gray}ResNet50 & \cellcolor{Gray}6 & \cellcolor{Gray}5 & \cellcolor{Gray}6 & \cellcolor{Gray}5 & \cellcolor{Gray}6 & \cellcolor{Gray}5 \\ 
 & \cellcolor{Gray}ResNet152 & \cellcolor{Gray}9 & \cellcolor{Gray}8 & \cellcolor{Gray}9 & \cellcolor{Gray}8 & \cellcolor{Gray}9 & \cellcolor{Gray}8 \\ 
\multirow{3}{*}{{\begin{minipage}{0.08\textwidth}
\centering
\includegraphics[width=1\linewidth]{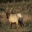}
\end{minipage}}} & ResNet18 & 1 & 1 & 1 & 1 & 1 & 1 \\ 
 & ResNet50 & 3 & 3 & 3 & 3 & 3 & 3 \\ 
 & ResNet152 & 6 & 5 & 6 & 5 & 6 & 5 \\ 
\multirow{3}{*}{{\begin{minipage}{0.08\textwidth}
\centering
\includegraphics[width=1\linewidth]{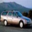}
\end{minipage}}} & \cellcolor{Gray} ResNet18 & \cellcolor{Gray}16 & \cellcolor{Gray}13 & \cellcolor{Gray}16 & \cellcolor{Gray}14 & \cellcolor{Gray}16 & \cellcolor{Gray}14 \\ 
 & \cellcolor{Gray} ResNet50 & \cellcolor{Gray}17 & \cellcolor{Gray}15 & \cellcolor{Gray}17 & \cellcolor{Gray}15 & \cellcolor{Gray}17 & \cellcolor{Gray}15 \\ 
 & \cellcolor{Gray} ResNet152 & \cellcolor{Gray}12 & \cellcolor{Gray}10 & \cellcolor{Gray}12 & \cellcolor{Gray}10 & \cellcolor{Gray}12 & \cellcolor{Gray}10 \\ \toprule
\end{tabular}
\caption{The maximum robustness radius estimated by \deeppac on CIFAR-10 dataset using different parameters, i.e. significance level $\eta$, error rate $\epsilon$, and the number of samples in the first learning phase $K^{(1)}$.}
\label{tab:stable}
\end{table}

\noindent\doublebox{
\begin{minipage}{0.94\linewidth}
{\bf Answer RQ2:} 
\deeppac shows good tolerance to different configurations of its parameters such as the error rate $\epsilon$, the significance level $\eta$, and the number of samples $K^{(1)}$.
\end{minipage}
}\\

\subsection{Scalability}\label{subsec:exp3}
Robustness verification is a well-known difficult problem on complex networks with high-dimensional data. Most qualitative verification methods meet a bottleneck in the size and structure of the DNN. The fastest abstract domain in ERAN is GPUPoly~\cite{deeppolygpu}, a GPU accelerated version of DeepPoly. The GPUPoly can verify a ResNet18 model on the CIFAR-10 dataset with an average time of 1\,021 seconds under the support of an Nvidia Tesla V100 GPU. To the best of our knowledge, ERAN does not support models on ImageNet, which makes it limited in real-life scenarios.
The statistical methods alleviate this dilemma and extend their use further.
The state-of-the-art PAC robustness verifier PROVERO needs to draw 737\,297 samples for VGG16 and 722\,979 samples for VGG19 on average for each verification case on ImageNet. The average running time is near 2208.9 seconds and 2168.9 seconds (0.003 seconds per sample) under the support of an Nvidia Tesla V100 GPU. We will show that \deeppac can verify the tighter PAC-model robustness on ImageNet with less samples and time on much larger ResNet50 models.

In this experiment, we apply \deeppac to the  start-of-the-art DNN with high resolution ImageNet images. The two ResNet50 networks are from the python package named ``robustness''~\cite{robustness}. 
We check PAC-model robustness of the two DNNs with the same radius $4$ (out of $255$). The first evaluation is on a subset of ImageNet images from 10 classes \cite{ImageNette}. 
The second one includes ImageNet images of all 1,000 classes and the untargeted score difference function is configured for \deeppac. 
To deal with ImageNet, the stepwise splitting mechanism in Sect.~\ref{sec:stepwise-splitting} is adopted.
An illustrating example of the stepwise splitting is given in Fig.~\ref{fig:expsplit}. As we expect, the splitting refinement procedure successfully identifies the significant features of a golf ball, i.e. the boundary and the logo. It maintains the accuracy of the learned model with much less running time.
The results are shown in Tab.~\ref{tab:Imagenet}.

\begin{figure}[t]
    \centering
        \centering
        \includegraphics[width=0.9\linewidth]{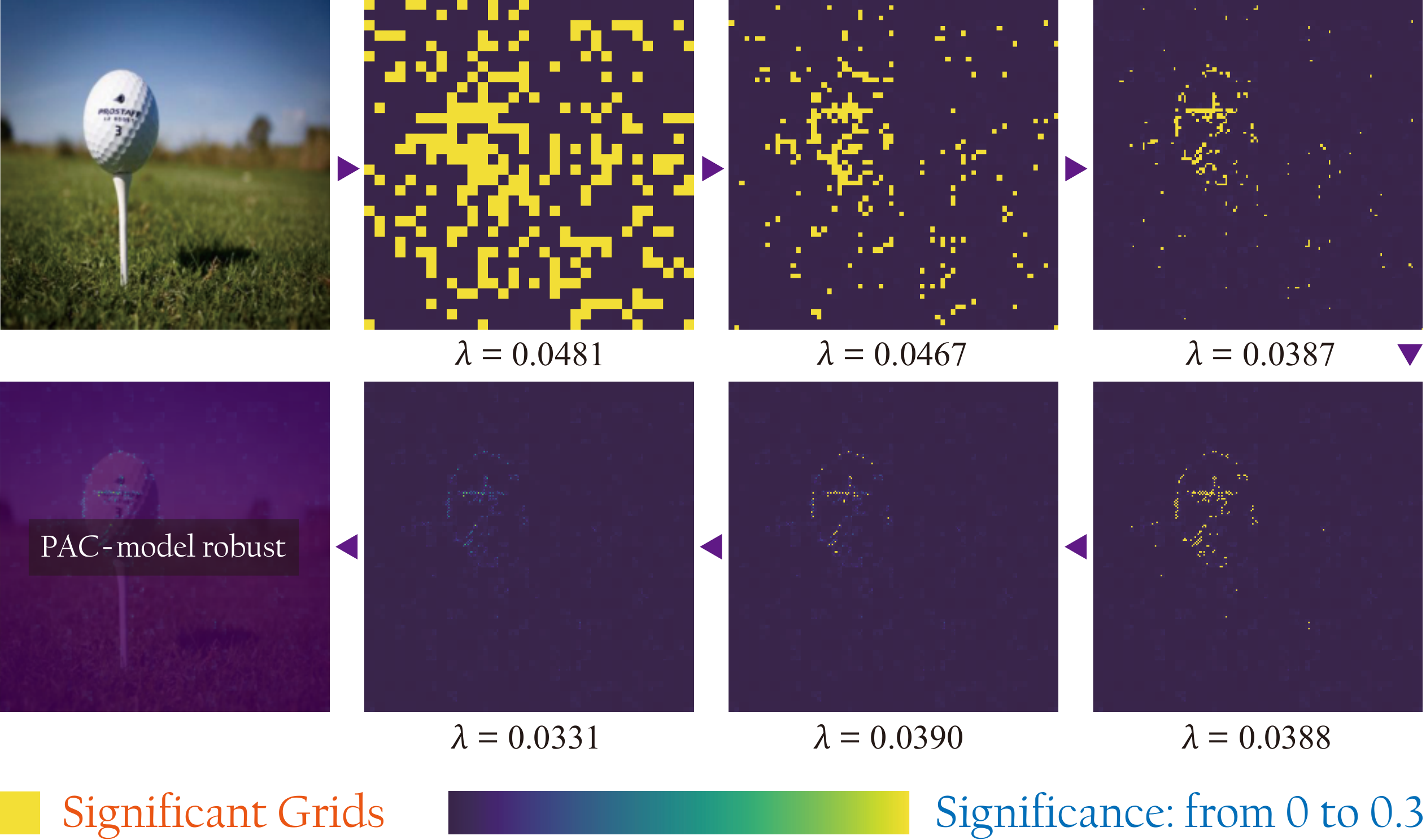}
    \caption{Stepwise splitting procedures of \deeppac, illustrated by heatmaps of grid significance. Top 25\% significant grids are colored yellow in the heatmap, which is split and refined iteratively. The margin $\lambda$ of different refinement stage is under the heatmap.}
    \label{fig:expsplit}
\end{figure}

For the 10-class setup, we evaluate the PAC-model robustness on 50 images and it takes less than 1800 seconds on each case. \deeppac finds out 30 and 29 cases PAC-model robust for ResNet50a and ResNet50b, respectively. 
Because the two models have been defensed, when we perform the PGD attack, only one adversarial examples were found for each model, which means that PGD gives no conclusion for the robustness evaluation on most cases under this setting. For the 1000-class dataset, the untargeted version of \deeppac has even better efficiency with the running time of less than 800 seconds each, which mainly benefits from reducing the score difference function to the untargeted one. \deeppac proves 10 and 6 out of 50 cases to be PAC-model robust on the 1000-class setup, respectively. For both setups, \deeppac uses 121\,600 samples to learn a PAC model effectively.

\begin{table}[h]
\begin{tabular}{llrrrr}
\toprule
Method & Network & Robust & Min & Max & Avg \\ \hline
\multirow{2}{*}{\begin{tabular}[c]{@{}l@{}}Targeted\\ (10 classes)\end{tabular}} & ResNet50a & 30/50 & 1736.5 & 1768.8 & 1751.8 \\ 
 & ResNet50b & 29/50 & 1722.1 & 1781.5 & 1746.5 \\ \hline
\multirow{2}{*}{\begin{tabular}[c]{@{}l@{}}Untargeted\\ (1000 classes)\end{tabular}} & ResNet50a & 10/50 & 779.2 & 785.3 & 781.7 \\ 
 & ResNet50b & 6/50 & 775.7 & 783.8 & 778.3 \\ \toprule
\end{tabular}
\caption{The performance of \deeppac analysing the two ResNet50 models for ImageNet. ``Robust'' represents the robustness rate. ``Min'', ``Max'', and  ``Avg'' are the minimum, maximum, and average of the running time (second), respectively.}
\label{tab:Imagenet}
\end{table}


\noindent\doublebox{
\begin{minipage}{0.94\linewidth}
    {\bf Answer RQ3:} The \deeppac robustness analysis scales well to complex DNNs with high-dimensional data like ImageNet,  
    which is not achieved by previous formal verification tools. 
    It shows superiority to \provero in both running time and the number of samples.
\end{minipage}
}\\

\subsection{Relation with Testing Prioritising Metric}

We also believe that there is a positive impact from practical DNN verification work like \deeppac on DNN testing. For example, the tool \deepgini uses Gini index, which measures the confidence of a DNN prediction on the corresponding input, to sort the testing inputs.
In Tab.~\ref{tab:testingcor}, we report the Pearson correlation coefficient between the \deepgini indices and the maximal robustness radii obtained by \deeppac, \eran and \provero from the experiment in Sect.~\ref{subsect:exp1}.


As in Tab.~\ref{tab:testingcor}, the maximum robustness radius is  correlated to the \deepgini index, a larger absolute value of the coefficient implies a stronger correlation. 
It reveals the data that has low prediction confidence is also prone to be lack robustness. From this phenomenon, we believe DeepGini can be also helpful in data selection for robustness analysis.
Interestingly, the maximum robustness radius computed by our \deeppac has higher correlations with \deepgini index on the CNNs, which are more complex, than on FNNs.
Furthermore, \deeppac shows the strongest correlation
on the CNNs trained with defense mechanisms, 
while the correlation between \provero or \eran and \deepgini is relatively weak on these networks. 
Intuitively, complex models with defense are expected to be more robust. Again, we regard this comparison result as the evidence from DNN testing to support the superior of \deeppac over other DNN verification tools.
From the perspective of testing technique, it is promising to combine these two methods for achieving test selection with guarantee.

\begin{table}[t]
\begin{tabular}{@{}l@{\qquad}r@{\qquad}r@{\qquad}r}
\toprule
Network & {\sc\textbf{DeepPAC}} & \eran & \provero \\ \hline
FNN1 & -0.3628 & -0.3437 & \textbf{-0.3968} \\ 
FNN2 & -0.4851 & -0.4353 & \textbf{-0.5142} \\ 
FNN3 & -0.4174 & -0.3677 & \textbf{-0.4223} \\ 
FNN4 & \textbf{-0.5264} & -0.4722 & -0.5234 \\ 
FNN5 & -0.4465 & \textbf{-0.6016} & -0.5916 \\ 
FNN6 & \textbf{-0.4538} & -0.2747 & -0.3949 \\ 
\hline
CNN1 &  -0.7340 & -0.7345 & \textbf{-0.8223} \\ 
CNN2 $\star$ & \textbf{-0.6482} & -0.6478 & -0.4527 \\ 
CNN3 $\star$ &\textbf{-0.7216} & -0.6728 & -0.5218 \\ 
CNN4    &-0.6035 & -0.6127 & \textbf{-0.7771} \\ 
CNN5 $\star$ &\textbf{-0.7448} & -0.6833 & -0.3874 \\ 
CNN6 $\star$ &\textbf{-0.6498} & -0.6094 & -0.4763 \\ \toprule
\end{tabular}
\caption{The Pearson correlation coefficient between the maximum robustness radius estimation and the DeepGini index. The DNNs are marked by ``$\star$'' if they are trained with defense mechanisms.}
\label{tab:testingcor}
\end{table}

\noindent\doublebox{
\begin{minipage}{0.94\linewidth}
{\bf Answer RQ4:}
The maximum robustness radius estimated by \deeppac, \eran, and \provero are all correlated to the \deepgini index, where \deeppac and \deepgini show the strongest correlation on robust models.
\end{minipage}
}\\

\subsection{Case Study: Verifying Cloud Service API}
To show the practicality of \deeppac, we apply it to analyse the robustness of  black-box models for real-world cloud services.
The case we study here is the image recognition API provided by Baidu AI Cloud\footnote{\url{https://ai.baidu.com/tech/imagerecognition/general}},
which accepts an image and returns a pair list in the form of $(\mathrm{label}_i,\mathrm{score}_i)$ to indicate the top classes the input recognised to be.
We use the image of a dandelion as the input, which is an official example in its illustration. 

By setting $\eta = 0.001$ and $\epsilon = 0.01$, we verify the PAC-model robustness for its top label ``dandelion'' within the radius of $5/255$. A total of 49,600 samples are utilised in the whole procedure. 
By \deeppac, we obtain the PAC-model of the difference function, but unfortunately, its maximal value in the input $L_\infty$ ball is larger than zero.
As an intermediate output, we generate a potential adversarial example via the PAC model.
By feeding it back into the model, we checked that it is a true adversarial example with ``sky'' as its top label (see Fig.~\ref{fig:apiresults}).

An interesting observation is that the labels output by the image recognition API may be not  independent. For instance, the class labels ``dandelion'' and ``plant'' may appear in the output list at the same time, and both of them can be considered correct labels. Therefore, we believe that in the future new forms of DNN robustness properties also need to be studied 
e.g., the sum of the output scores for the correct labels (``dandelion'' and ``plant'') should be larger than some threshold.
\deeppac is a promising tool to cope with these emerging challenges when considering real-world applications of DNN robustness analysis, by conveniently adjusting its difference function. 

\begin{figure}[t]
    \centering
    \subfloat{%
        \includegraphics[width=0.45\linewidth]{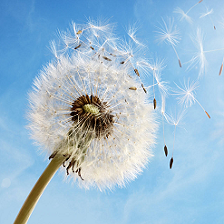}}
    \qquad
    \subfloat{%
        \includegraphics[width=0.45\linewidth]{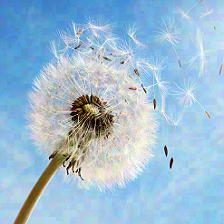}}
    \caption{The original image (left) gain the score $(\mathrm{dandelion}:0.758, \mathrm{sky}:0.600)$, and the adversarial example gain the score $(\mathrm{sky}:0.791, \mathrm{dandelion}:0.621)$.}
    \label{fig:apiresults}
\end{figure}

\input{related}

\section{Conclusion and Future Work} \label{sec:conclusion}
We propose \deeppac, a method based on model learning to analyse the PAC-model robustness of DNNs in a local region. With the scenario optimisation technique, we learn a PAC model which approximates the DNN within a uniformly bounded margin with a PAC guarantee. With the learned PAC model, we can verify PAC-model robustness properties under specified confidence and error rate.
Experimental results confirm that \deeppac scales well on large networks, and is suitable for practical DNN verification tasks. 
As for future work, we plan to learn more complex PAC models rather than the simple affine models, and we are particularly interested in exploring the combination of practical DNN verification by \deeppac and DNN testing methods following the preliminary results.





\section*{Acknowledgements}
This work has been partially supported by Key-Area Research and Development Program of Guangdong Province (Grant No. 2018B010107004), Guangzhou Basic and Applied Basic Research Project (Grant No. 202102021304), 
National Natural Science Foundation of China (Grant No. 61836005), and Open Project of Shanghai Key Laboratory of Trustworthy Computing (Grant No. OP202001).

\bibliographystyle{ACM-Reference-Format}
\bibliography{all.bib}


\end{document}

%% file: introduction.tex
\section{Introduction}
\label{sec:introduction}

Deep neural networks (DNNs) are now widely deployed in many applications such as image classification,  game playing, and the recent scientific discovery on predictions of protein structure~\cite{alphafold}. Adversarial robustness of a DNN plays the critical role for its trustworthy use. This is especially true for for safety-critical applications such as self-driving cars~\cite{selfdriving}. Studies have shown that even for a DNN with high accuracy, it can be fooled easily by carefully crafted \emph{adversarial inputs} \cite{SZSBEGF2014}.
This motivates research on verifying DNN robustness properties, i.e., the prediction of the DNN remains the same after bounded perturbation on an input.
As the certifiable criterion before deploying a DNN, the robustness radius should be estimated or the robustness property should be verified. 



In this paper, we propose a practical framework for analysing robustness of DNNs.
The main idea is to learn an affine model which abstracts local behaviour of a DNN and use the learned model (instead of the original DNN model) for robustness analysis. Different from model abstraction methods like \cite{abstractioncav20,atva2020}, our learned model is not a strictly sound over-approximation, but it varies from the DNN uniformly within a given margin subject to some specified significance level and error rate. We call such a model the probably approximately correct (PAC) model.

There are several different approaches to estimating the maximum robustness radius of a given input for the DNN, including formal verification, statistical analysis, and adversarial attack. In the following, we will first briefly explain the pros and cons of each approach for and its relation with our method. Then, we will highlight the main contributions in this paper. 

\paragraph{Bound via formal verification is often too conservative} A DNN is a complex nonlinear function and formal verification tools \cite{marabou,deepz,deeppoly,cnncert,NNV,prodeep,deepsrgrextended} can typically handle DNNs with hundreds to thousands of neurons. This is dwarfed by the size of modern DNNs used in the real world, 
such as the ResNet50 model~\cite{resnet} used in our experiment with almost $37$ million hidden neurons. 
The advantage of formal verification is that its resulting robustness bound is guaranteed, but the bound is also often too conservative. For example, the state-of-the-art formal verification tool \eran is based on abstract interpretation \cite{deeppoly} that over-approximates the computation in a DNN using computationally more efficient abstract domains. If the \eran verification succeeds, one can conclude that the network is locally robust; otherwise, due to its over-approximation, no conclusive result can be reached and the robustness property may or may not hold.

\commentout{
In the formal verification community, abstraction-based techniques have been successfully applied to non-trivial networks  to check robustness properties. The idea is to encode computations of neurons as constraints, and check the robustness property with SMT-based solvers, abstract interpretation, or counterexample guided abstraction refinement. If the verification result is affirmative, one can conclude that the network is locally robust. Due to over-approximations in the abstraction-based approach, the result may be inconclusive otherwise.
Although model abstraction of DNNs proposed recently in \cite{abstractioncav20,atva2020} can verify relatively large networks, in which the abstraction models have a strict soundness guarantee, these abstraction-based methods are still limited in industrial scale applications, where neural networks may be characterised by millions or even billions of parameters.
}

\paragraph{Estimation via statistical methods is often too large}
If we weaken the robustness condition by allowing a small error rate on the robustness property, it becomes a probabilistic robustness (or quantitative robustness) property.
Probabilistic robustness characterises the local robustness in a way similar to the idea of the label change rate in mutation testing for DNNs ~\cite{detection1,detection2}. 
In \cite{statistical19,DBLP:conf/icml/WengCNSBOD19,baluta2021scalable,quanti4,quanti5,quanti6,BayesianIJCAI}, statistical methods are proposed to evaluate local robustness with a probably approximately correct (PAC) guarantee. That is, with a given confidence, the DNN satisfies a probabilistic robustness property, and we call this \emph{PAC robustness}. However, as we are going to see in the experiments (Section \ref{sec:experiment}), the PAC robustness estimation via existing statistical methods is often unnecessarily large. In this work, our method significantly improves the PAC robustness bound, without loss of confidence or error rate.

\paragraph{Bound via adversarial attack has no guarantee}
Adversarial attack algorithms apply various search heuristics based on e.g., gradient descent or evolutionary techniques for generating adversarial inputs~\cite{CW-Attacks,pgd,GenAttack,zhang2020walking}. These methods may be able to find adversarial inputs efficiently, but are not able to provide any soundness guarantee. While the adversarial inputs found by the attack establish an upper bound of the DNN local robustness, it is not known whether there are other adversarial inputs within the bound. Later, we will use this upper bound obtained by adversarial attack, together with the lower bound proved by the formal verification approach discussed above, as the reference for evaluating the quality of our PAC-model robustness results, and comparing them with the latest statistical method.


\paragraph{\textbf{Contributions}} 
We propose a novel framework of PAC-model robustness verification for DNNs. 
Inspired by the scenario optimisation technique in robust control design, we give an algorithm to learn an affine PAC model for a DNN.
This affine PAC model captures local behaviour of the original DNN. It is simple enough for efficient robustness analysis, and its PAC guarantee ensures the accuracy of the analysis. 
We implement our algorithm in a prototype called \deeppac.  
We extensively evaluate \deeppac with 20 DNNs on three datasets. \deeppac outperforms the state-of-the-art statistical tool \provero with less running time, fewer samples and, more importantly, much higher precision. \deeppac can assess the DNN robustness faithfully when the formal verification and existing statistical methods fail to generate meaningful results.

\commentout{
 \begin{itemize}
 \item {\bf Black-box.} 
 In many application scenarios, the DNN is a black box in nature due to either a lack of an adequate knowledge on its specific structure and parameter or its high complexity and large scale. 
 Existing white-box methods fail to work with these black-box DNNs. However, a black-box method, which only uses inputs and outputs of DNNs, works. 
 \item {\bf Model learning.} Sampling can help gain knowledge of a black-box system, but the knowledge gained only from samples is far from inferring a robustness property of a black-box system well, especially when the dimensionality of the input is large.
To learn an abstraction model facilitates the fight against this issue. The essential difference between our model-driven method and previous sampling-based methods like \cite{statistical19,DBLP:conf/icml/WengCNSBOD19,DBLP:journals/corr/abs-2002-06864,anderson2020certifying,anderson2020datadriven} is that we learn an abstraction with an error bound under a PAC guarantee, which captures local behaviour of the original DNN and thus can provide more insights into the DNNs and facilitates more accurate analysis of its robustness.
 \end{itemize}

The method of learning an DNN is inspired by the scenario optimisation technique in robust control design, which has been recently applied in PAC learning for safety verification of continuous-time black-box dynamical systems~\cite{xuebaipaper}. 
Scenario optimisation considers only a sample of constraints to provide a solution for robust optimisation problems.
The learned model can be used to offer a PAC guarantee for the robustness property. In this work, we focus on the investigation of widely studied $L_\infty$ robustness based on scenario optimisation, but the proposed method can also be extended to other robustness.
Moreover, the counterexamples of the learned model can potentially attack the original DNN. Exploiting the linearity of the learned model, we can directly compute a potential counterexample, which will be further identified as real or spurious by running the original DNN.
}

\commentout{
We have implemented our algorithm in a prototypical tool DeepPAC and tested it on a number of benchmarks. Experimental results demonstrate that our method is promising. For the MNIST dataset, compared to white-box tools, DeepPAC can verify larger robustness radii with very high confidence and low error rate; for the CIFAR-10 dataset, DeepPAC can handle ResNet networks with more than $6.5$~million neurons, which cannot be processed by existing sound verification tools. For counterexample generation, the visualisation of the coefficients of the learned models and the counterexamples captures the essential differences between the original and the attacking images and reflects the direction of our attacks on the level of pixels. We observe that DeepPAC can generate adversarial examples very close to the decision boundary. Our method is competitive to existing black-box and white-box attacking methods.

}

\paragraph{Organisation of the paper} The rest of this paper is organized as follows. In Sect.~\ref{sec:preliminaries}, we first introduce the background knowledge. We then formalize the novel concept PAC-model robustness in Sect.~\ref{sec:pac-model-robustness}.  The methodology is detailed in Sect.~\ref{sec:abstraction}. Extensive experiments have been conducted in Sect.~\ref{sec:experiment} for evaluating \deeppac. We discuss related work in Sect.~\ref{sec:related_work} and conclude our work in Sect.~\ref{sec:conclusion}.

%% file: preliminary.tex
\section{Preliminary} \label{sec:preliminaries}
In this section, we first recall the background knowledge on the DNN and its local robustness properties.
Then, we introduce the scenario optimization method that will be used later.
In this following context, we denote $x_i$ as the $i$th entry of a vector $\boldsymbol{x} \in \mathbb R^m$. For $\bm{x} \in \mathbb R^m$ and $\lambda \in \mathbb R$, we define $\bm{x}+\lambda$ as $(x_0+\lambda,\ldots,x_m+\lambda)^\top$. Given $\bm{x},\bm y \in \mathbb R^m$, we write $\bm{x}\leq \bm{y}$ if $x_i\leq y_i$ for $i = 1,\ldots,m$. We use $\bm{0}$ to denote the zero vector. 
For $\boldsymbol{x} \in \mathbb{R}^m$, its $L^\infty$-norm is defined as $\|\boldsymbol{x}\|_\infty:=\max_{1 \le i \le m}|x_i|$.
We use the notation $B({\hat{\bm{x}}},r):=\{\bm{x} \in \mathbb{R}^m \mid \|\boldsymbol{x}-{\hat{\bm{x}}}\|_\infty\leq r\}$ to represent the closed $L^\infty$-norm ball with the center ${\hat{\bm{x}}} \in \mathbb{R}^m$ and radius $r>0$. 

\subsection{DNNs and Local Robustness}
\label{sec:dnn}
A deep neural network can be characterized as a function $\bm{f}:\mathbb{R}^m\to\mathbb{R}^n$ with $\bm{f}=(f_1,\ldots,f_n)^{\top}$, where $f_i$ denotes the function corresponding to the $i$th output. 
For classification tasks, a DNN labels an input $\bm{x}$ with the output dimension having the largest score, denoted by 
$C_{\bm{f}}(\boldsymbol{x}):=\arg\max_{1\leq i\leq n}
f_i(\bm{x})$.
A DNN is composed by multiple layers: the input layer, followed by several hidden layers and an output layer in the end. 
A hidden layer applies an affine function or a non-linear activation function on the output of previous layers. The function $\bm{f}$ is the composition of the transformations between layers.  

\begin{example}
\label{example:dnn}
We illustrate a fully connected neural network (FNN), where each node (i.e., neuron) is connected with the nodes from the previous layer. Each neuron has a value that is calculated as the weighted sum of the neuron values in the previous layer, plus a bias. For a hidden neuron, this value is often followed by an activation function e.g., a ReLU function that rectifies any negative value into $0$.
In Fig.~\ref{fig:smallnet}, the FNN characterizes a function
  $\bm{f}:\mathbb{R}^2\rightarrow\mathbb{R}^2$.
  The weight and bias parameters are highlighted on the edges and the nodes respectively.
  For an input $\bm{x}=(x_1,x_2)^{\top}\in[-1,1]^2$, we have $\bm{f}(\x)=(f_1(\x),f_2(\x))^{\top}$.
\begin{figure}[h]
    \centering
    \begin{tikzpicture}[shorten >=1pt,->,draw=black, node distance=\layersep]
        \tikzstyle{every pin edge}=[<-,shorten <=1pt]
        \tikzstyle{neuron}=[circle,fill=black,scale=0.6]
        \tikzstyle{input neuron}=[neuron]
        \tikzstyle{output neuron}=[neuron]
        \tikzstyle{hidden neuron}=[neuron]
        \tikzstyle{annot} = [text width=4em, text centered, font=\fontsize{8}{6}]
    
        \foreach \name / \y in {1,...,2}
            \node[input neuron, pin=left:Input \y] (I-\name) at (0,-\y*1.5) {};
    
        \foreach \name / \y in {1,...,2}
            \path[yshift=0cm]
                node[hidden neuron] (H-\name) at (\layersep,-\y*1.5) {};
    
        \node[output neuron,pin={[pin edge={->}]right:Output 1}] (O1) at (\layersep*2,-1.5) {};
        \node[output neuron,pin={[pin edge={->}]right:Output 2}] (O2) at (\layersep*2,-3) {};
    
        \path (I-1) edge[above left,] node{$3$} (H-1);
        \path (I-1) edge[above left,sloped] node{$5$} (H-2);
        
        \path (I-2) edge[below left,sloped] node{$-10$} (H-1);
        \path (I-2) edge[below left] node{$-4$} (H-2);

        \path (H-1) edge[above left,sloped] node{$3$} (O1);
        \path (H-2) edge[below left,sloped] node{$1$} (O1);
        
        \path (H-1) edge[above left,sloped] node{$9$} (O2);
        \path (H-2) edge[below left,sloped] node{$7$} (O2);
    
        \node[annot,above of=I-1,color=black,node distance=0.5cm]{$[-1,1]$};
        \node[annot,above of=H-1,color=black,node distance=0.5cm] (H-1n){$-9$};
        \node[annot,above of=O1,color=black,node distance=0.5cm] (On){$14$};
        
        \node[annot,below of=I-2,color=black,node distance=0.5cm]{$[-1,1]$};
        \node[annot,below of=H-2,color=black,node distance=0.5cm] (H-2n){$-10$};
        \node[annot,below of=O2,color=black,node distance=0.5cm] (On){$-10$};
        
    \end{tikzpicture}
    \caption{An FNN with two input neurons, two hidden neurons and two output neurons.} \label{fig:ex_fnn}
    \label{fig:smallnet}
\end{figure}
\end{example}



For a certain class label $\ell$,
we define the \emph{targeted score difference function} $\bm{\Delta}$
as
\begin{equation} \label{eq:difference}
\bm{\Delta} (\bm{x}) = (f_1(\bm{x}) - f_\ell(\bm{x}), \ldots, f_n(\bm{x}) - f_\ell(\bm{x}))^{\top}\;.
\end{equation}
Straightforwardly, this function measures the difference between the score of the targeted label and other labels.
For simplicity, we ignore the entry $f_\ell(\bm{x}) - f_\ell(\bm{x})$ and regard the score difference function $\bm{\Delta}$ as a function from $\mathbb R^m$ to $\mathbb R^{n-1}$.
For any inputs $\hat x$ with the class label $\ell$, it is clear that $\bm{\Delta} (\hat \x) < \bm{0}$ if the classification is correct.
For simplicity, when considering an $L^\infty$-norm ball
with the center $\hat \x$, we denote by $\bm{\Delta}$
the difference score function with respect to the
label of $\hat \x$.
Then robustness property of a DNN can therefore be defined as below.


\begin{definition}[DNN robustness]
\label{def:localrobustness}
Given a DNN $\bm{f}:\mathbb{R}^m \to \mathbb{R}^n$, an input ${\hat{\bm{x}}}\in \mathbb{R}^m$, and $r>0$, we say that $\bm{f}$ is (locally) \emph{robust} 
in $B({\hat{\bm{x}}},r)$
if for all $\boldsymbol{x} \in B({\hat{\bm{x}}},r)$, we have
$\bm{\Delta}(\x) <\bm{0}$.
\end{definition}
Intuitively, local robustness ensures the consistency of the behaviour of a given input under certain perturbations. 
An input $\x' \in B(\hat\x,r)$ that destroys the robustness (i.e. $\bm{\Delta}(\x')\ge \bm{0}$) is called an \emph{adversarial example}. 
Note that this property is very strict so that
the corresponding verification problem 
is NP-complete, 
and the exact maximum robustness radius
cannot be computed efficiently except for very small DNNs.
Even estimating a relatively accurate lower bound is difficult and existing sound methods cannot scale to the state-of-the-art DNNs.
In order to perform more practical DNN robustness analysis, the property is relaxed by allowing some errors in the sense of probability.
Below we recall the definition of \emph{PAC robustness}~\cite{baluta2021scalable}.



\begin{definition}[PAC robustness]
\label{def:PACrobustness}
Given a DNN $\bm{f}:\mathbb{R}^m \to \mathbb{R}^n$, an $L_\infty$-norm ball $B(\hat{\bm x},r)$, a probability measure $\mathbb P$ on $B(\hat{\bm x},r)$, a significance level $\eta$, and an error rate $\epsilon$, the DNN $\bm{f}$ is $(\eta,\epsilon)$-PAC robust 
in $B({\hat{\bm{x}}},r)$
if
\begin{equation}
 \mathbb{P}( \bm{\Delta}(\bm{x}) < 0) \geq 1-\epsilon
\end{equation}
with confidence $1-\eta$.
\end{definition}


PAC robustness is an statistical relaxation and extension of DNN robustness in Def.~\ref{def:localrobustness}. It essentially only focuses on the input samples, but mostly ignores the behavioral nature of the original model. When the input space is of high dimension, the boundaries between benign inputs and adversarial inputs will be extremely complex and the required sampling effort will be also challenging. Thus, an accurate estimation of PAC robustness is far from trivial. This motivates us to innovate the PAC robustness with PAC-model robustness in this paper (Sect.~\ref{sec:pac-model-robustness}).

\commentout{
the number of adversarial examples that can be tolerated is very large.
We can foresee that using this definition to analyze the robustness of neural networks, especially to estimate the maximum robustness radius, is very inaccurate.
The reason for this situation is that the definition of PAC robustness essentially only focuses on the sample, but completely ignores the behavioral nature of the original model.
}

\subsection{Scenario Optimization}
Scenario optimization is another motivation for \deeppac. It has been successfully used in robust control design for solving a class of optimization problems in a statistical sense, by only considering a randomly sampled finite subset of infinitely many convex constraints \cite{remove_unique,scenario_theorem}.

Let us consider the following optimization problem: 
\begin{equation}
    \label{robustopt}
    \begin{split}
        & \min\limits_{\boldsymbol{\gamma}\in \Gamma \subseteq \mathbb{R}^m}\boldsymbol{b}^\top\boldsymbol{\gamma}\\
        s.t.\;& f_{\boldsymbol{\omega}}(\boldsymbol{\gamma})\leq 0,\;\forall \boldsymbol{\omega}\in \Omega ,
    \end{split}
\end{equation}
where $f_{\boldsymbol{\omega}}$ is a convex and continuous function of the $m$-dimensional optimization variable $\boldsymbol{\gamma}$ for every $\boldsymbol{\omega}\in \Omega$, and both $\Omega$ and $\Gamma$ are convex and closed.
In this work, we also assume that $\Omega$ is bounded.
In principle, it is challenging to solve \eqref{robustopt}, as there are infinitely many constraints.
Calafiore et al.~\cite{remove_unique} proposed the following scenario approach to solve \eqref{robustopt} with a PAC guarantee.

\begin{definition}
    \label{scenariodef}
   Let $\mathbb P$ be a probability measure on $\Omega$. The scenario approach to handle the optimization problem~\eqref{robustopt} is to solve the following problem. We extract $K$ independent and identically distributed (i.i.d.) samples $(\boldsymbol{\omega}_i)_{i=1}^K$ from $\Omega$ according to the probability measure $\mathbb P$:
    \begin{equation}
    \label{scenarioopt}
    \begin{split}
        & \min\limits_{\boldsymbol{\gamma}\in \Gamma \subseteq \mathbb{R}^m}\boldsymbol{b}^\top\boldsymbol{\gamma}\\
        \mathrm{s.t.}\;&\bigwedge_{i=1}^{K} f_{\boldsymbol{\omega}_i}(\boldsymbol{\gamma})\leq 0.
    \end{split}
\end{equation}
\end{definition}
The scenario approach relaxes the infinitely many constraints in \eqref{robustopt} by only considering a finite subset containing $K$ constraints. In \cite{remove_unique}, a PAC guarantee, depending on $K$,
between the scenario solution in \eqref{scenarioopt} and its original optimization in \eqref{robustopt} is proved.
This is further improved by 
\cite{scenario_theorem} in reducing the number of samples $K$.
Specifically, the following theorem establishes a condition on $K$ for \eqref{scenarioopt} which assures that its solution satisfies the constraints in \eqref{robustopt} statistically.

\begin{theorem}[\cite{scenario_theorem}]\label{scenariotheorem}
    If (\ref{scenarioopt}) is feasible and has a unique optimal solution $\boldsymbol{\gamma}^*_K$, and
    \begin{equation}
        \label{eq:scenariotheorem}
        \epsilon\geq \frac{2}{K}(\ln\frac{1}{\eta}+m),
    \end{equation}
    where $\epsilon$ and $\eta$ are the pre-defined error rate and the significance level, respectively, then with confidence at least $1-\eta$, the optimal $\boldsymbol{\gamma}^*_K$ satisfies all the constraints in $\Omega$ but only at most a fraction of probability measure $\epsilon$, i.e., $\mathbb P (f_{\boldsymbol{\omega}}(\boldsymbol{\gamma}_K^*)> 0)\leq \epsilon$.
\end{theorem}

In this work, we set $\mathbb P$ to be 
the uniform distribution on the $\Omega$ set in \eqref{robustopt}. It is worthy mentioning that Theorem~\ref{scenariotheorem} still holds even if the uniqueness of the optimal $\boldsymbol{\gamma}^*_K$ is not required, since a unique optimal solution can always be obtained by using the Tie-break rule~\cite{remove_unique} if multiple optimal solutions exist. 

 The scenario optimization technique has been exploited in the context of black-box verification for continuous-time dynamical systems in \cite{xuebaipaper}. We will propose an approach based on scenario optimization to verify PAC-model robustness in this paper.

%% file: pac-model-robustness.tex
\section{PAC-Model Robustness}
\label{sec:pac-model-robustness}
The formalisation of the novel concept PAC-model robustness is our first contribution in this work and it is the basis for developing our method.  We start from defining a \emph{PAC model}. Let $\mathcal F$ be a given set of high dimensional real functions (like affine functions).

\begin{definition}[PAC model]\label{def:PACmodel}
Let $\bm{g}:\R^m\to \R^{n}$, $B \subseteq \mathbb R^m$ and $\mathbb P$ a probability measure on $B$. 
Let $\eta,\epsilon\in (0,1]$ be the given error rate and significance level, respectively.
Let $\lambda\ge 0$ be the margin.
A function $\widetilde{\bm{g}}:B \to \R^{n} \in \mathcal F$ is a PAC model of $\bm{g}$ on $B$ w.r.t. $\eta$, $\epsilon$ and $\lambda$, denoted by $\widetilde{\bm{g}} \approx_{\eta,\epsilon,\lambda} \bm{g}$, if
\begin{equation}\label{eq:abstConf}
    \mathbb P( || \widetilde{\bm{g}}(\bm{x})-\bm{g}(\bm{x})||_\infty \le \lambda ) \ge 1-\epsilon ,
\end{equation}
with confidence $1-\eta$. 
\end{definition}
In Def.~\ref{def:PACmodel}, we define a PAC model $\widetilde{\bm{g}}$ as an approximation of the original model $\bm{g}$  with two parameters $\eta$
and $\epsilon$ which bound the maximal significance level and
the maximal error rate for the PAC model, respectively.
Meanwhile, there is another parameter $\lambda$
that bounds the margin between the PAC model and the original model. Intuitively, the difference between a PAC model and the original one is bounded under the given error rate $\epsilon$ and significance level $\eta$. 

For a DNN $\bm f$, if its PAC model $\widetilde{\bm f}$ with the corresponding margin is robust, then $\bm f$ is PAC-model robust. Formally, we have the following definition.
\begin{definition}[PAC-model robustness]
\label{def:pacmodelrobustness}
Let $\bm{f}:\mathbb{R}^m \to \mathbb{R}^n$ be a DNN and $\bm{\Delta}$ the corresponding score difference. Let $\eta,\epsilon\in (0,1]$ be the given error rate and significance level, respectively.
The DNN $\bm{f}$ is $(\eta,\epsilon)$-PAC-model robust in $B({\hat{\bm{x}}},r)$,
if there exists a PAC model
$\widetilde{\bm{\Delta}} 
\approx_{\eta,\epsilon,\lambda} 
{\bm{\Delta}}$
such that for all $\boldsymbol{x}\in B(\hat{\bm{x}},r)$,
\begin{equation}
    \label{eq:pacmodelrobustness}
    \widetilde{\bm{\Delta}}(\bm{x})+\lambda < \bm{0}.
\end{equation}
\end{definition}
We remind that $\bm{\Delta}$ is the score difference function 
measuring the difference between the score of the targeted label and other labels. A locally robust DNN requires that $\bm{\Delta(x)}<\bm{0}$, and a PAC-model robust DNN requires the PAC upper bound of $\bm{\Delta}$, i.e. $\widetilde{\bm{\Delta}}(\bm{x})+\lambda$, is always smaller than $\bm{0}$.


In Fig.~\ref{fig:property}, we illustrate the property space of PAC-model robustness, by using the parameters $\eta$, $\epsilon$ and $\lambda$. The properties on the $\lambda$-axis are
exactly the strict robustness
since $\Delta(\bm{x})$ is now strictly
upper-bounded by $\widetilde\Delta(\bm{x})+\lambda$.
Intuitively, for fixed $\eta$ and $\epsilon$, a smaller margin $\lambda$ implies a better PAC approximation $\widetilde\Delta(\bm{x})$ of 
 the original one $\Delta(\bm{x})$ and
indicates that the PAC-model robustness is closer to the (strict) robustness property of the original model.
To estimate the maximum robustness radius more accurately, we intend to compute a PAC model with the margin $\lambda$ as small as possible. 
Moreover, the proposed PAC-model robustness is stronger than PAC robustness, which is proved by the following proposition.  
 \begin{proposition} \label{prop:PACrobustness}
If a DNN $\bm f$ is $(\eta,\epsilon)$-PAC-model robust in $B({\hat{\bm{x}}},r)$, then it is $(\eta,\epsilon)$-PAC robust in $B({\hat{\bm{x}}},r)$.
 \end{proposition}
 \begin{proof}
With confidence $1-\eta$ we have 
    \[
    \begin{aligned}
\qquad \mathbb{P}(\bm{\Delta}(\bm{x}) & \leq \bm{0})
    \geq \mathbb{P}(\bm{\Delta}(\bm{x}) 
    \leq \bm{\widetilde\Delta}(\bm{x})+\lambda)\\
    & \geq\mathbb{P}(||\bm{\widetilde\Delta}(\bm{x})-\bm{\Delta}(\bm{x})||_\infty\leq \lambda)
    \geq1-\epsilon,
        \end{aligned}
\]
which implies that $\bm f$ is $(\eta,\epsilon)$-PAC robust in $B({\hat{\bm{x}}},r)$.
\end{proof}

In this work, wo focus on the following problem:
\begin{quote}
    Given a DNN $\bm f$, an $L_\infty$-norm ball $B(\hat {\bm x},r)$, a significance level $\eta$, and an error rate $\epsilon$, we need to determine whether $\bm f$ is $(\eta,\epsilon)$-PAC-model robust.
\end{quote}
Before introducing our method, we revisit PAC robustness (Def.~\ref{def:PACrobustness}) in our PAC-model robustness theory. 
Statistical methods like \cite{baluta2021scalable} infer PAC robustness from samples and their classification output in the given DNN. 
In our PAC-model robustness framework, these methods simplify the model to a function $\xi:B(\hat {\bm x},r) \to \{0,1\}$, where $0$ refers to the correct classification result and $1$ a wrong one, and  infer the PAC-model robustness with the constant function $\widetilde\xi(\bm x) \equiv 0$ on $B(\hat {\bm x} ,r)$ as the model. 
In \cite{anderson2020certifying}, the model is modified to a constant score difference function $\widetilde{\bm \Delta} \equiv c$.
These models are too weak to describe the behaviour of a DNN well.
It can be predicted that, if we learn a PAC model with an appropriate model, the obtained PAC-model robustness property will be more accurate and practical, and this will be demonstrated in our experiments.


\begin{figure}[t]
    \centering
        \centering
        \includegraphics[width=0.6\linewidth]{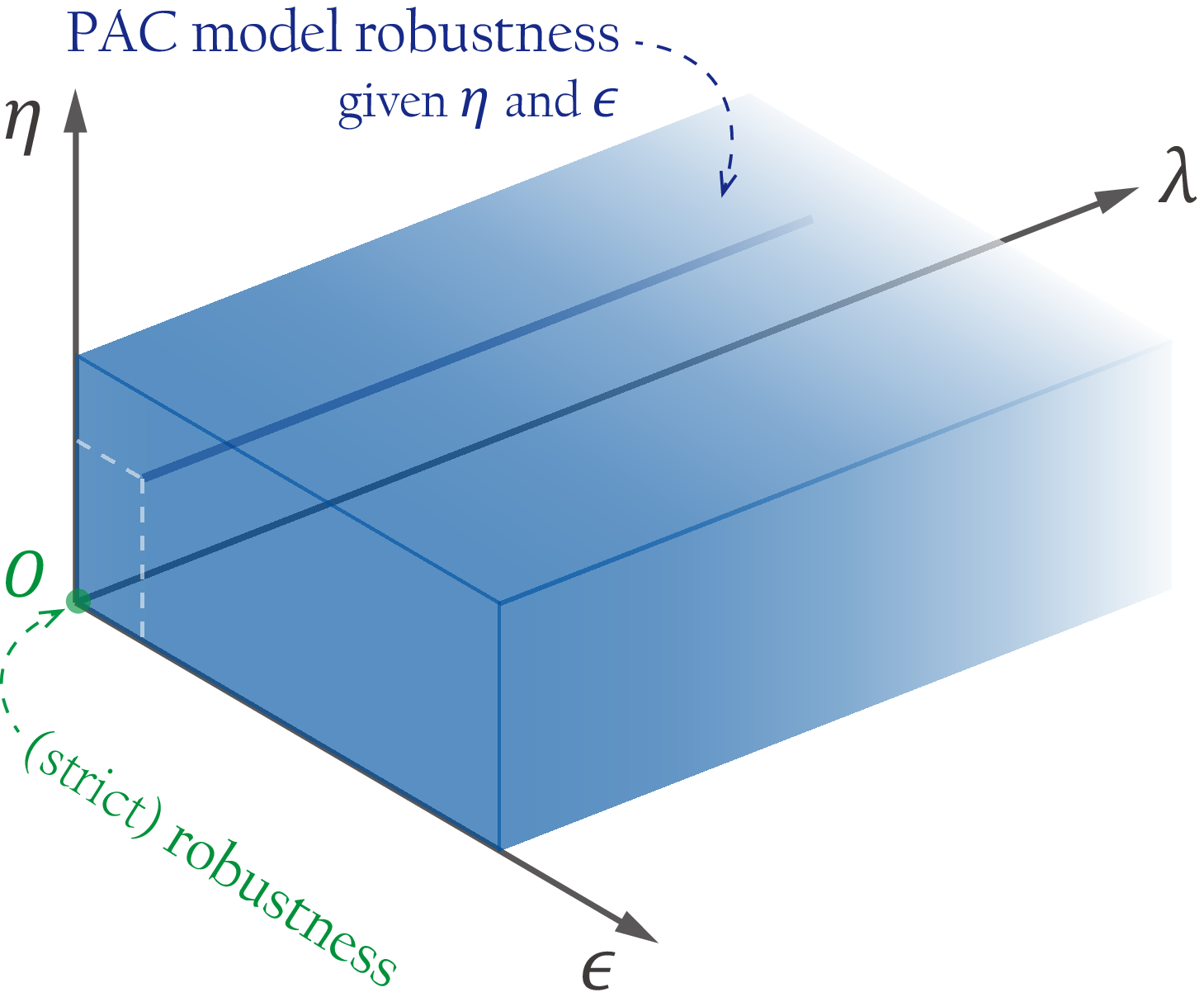}
    \caption{Property space of PAC-model robustness.}
    \label{fig:property}
\end{figure}

%% file: method.tex
\section{Methodology} \label{sec:abstraction}
In this section, we present our method for analysing the PAC-model robustness of DNNs. 
The overall framework is shown in Fig.~\ref{fig:framework}.
In general, our method comprises of three stages:
sampling, learning, and analysing.
\begin{itemize}
    \item[S1:] We sample the input region $B(\hat\x,r)$ and obtain the corresponding values of the score difference function $\bm{\Delta}$.
    \item[S2:] We learn a PAC model $\widetilde{\bm{\Delta}}(x)\approx_{\eta,\epsilon,\lambda}\bm{\Delta}(x)$ of the score difference function from the samples.
    \item[S3:] We analyse whether $\widetilde{\bm{\Delta}}(x)+\lambda$ is always negative in the region $B(\hat\x,r)$ by computing its maximal values.
\end{itemize}

From the description above, 
we see it is a 
black-box method since
we only use the 
samples in the neighbour
and their corresponding outputs
to construct the PAC model.
The number of samples is independent of the structure and the size of original models,
which will bring the good scalability and efficiency.
Moreover, we are essentially reconstructing a proper model to depict the local behavior of the original model.
Compared with the statistical methods,
the PAC model can potentially extract more information from the score differences of these samples, which supports us to obtain more accurate results.

Note that our framework is constructive, and the PAC model and its maximal points in the region will be constructed explicitly during the analysis. Then, we can obtain the maximal values of the PAC model, and infer that the original DNN satisfies the PAC-model robustness when all maximal values are negative. Thus, \deeppac can be considered as a sound approach to verify the PAC-model robustness.

\begin{figure}[ht!]
    \centering
        \centering
        \includegraphics[width=0.98\linewidth]{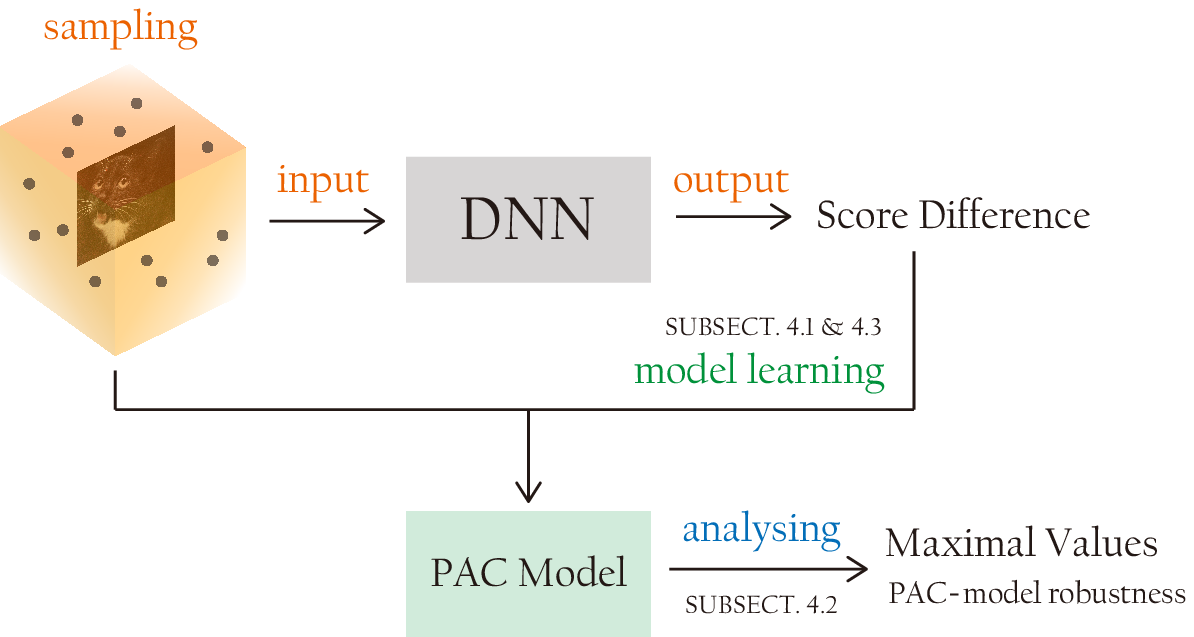}
    \caption{Framework of PAC-model robustness analysis base on model learning}
    \label{fig:framework}
\end{figure}

\subsection{Learning a PAC Model} \label{subsec:learn_original}
To obtain a PAC model of the original score
difference function $\bm{\Delta({\x})}$,
we first create a function template, and then determine its parameters by model learning from the samples.
Hereafter, we set $\mathcal F$ to be the set of affine functions, and consider the PAC model $\widetilde{\bm{\Delta}}(\bm{x})$ to be an affine function with bounded coefficients.
A reason for choosing an affine template is that the behaviours of a DNN in a small $L^\infty$-norm ball $B(\hat{\bm{x}},r)$ are very similar to some affine function \cite{lime}, due to the almost everywhere differentiability of DNNs.
In other words, an affine function can approximate the original model well enough in most cases
to maintain the accuracy of our robustness analysis.
Specifically, for the $i$th dimension of the DNN output layer, we set 
$\widetilde\Delta_i(\x)=\c_i^\top \x=c_{i,0}+c_{i,1}x_1+\cdots+c_{i,m}x_m$.
With extracting a set of $K$ independent and identically distributed samples $\hat X \subseteq B(\hat{\bm{x}},r)$, we construct the following optimisation problem for learning the affine PAC model $\widetilde{\bm{\Delta}}(\bm{x})$.
\begin{equation}\label{eq:linearP}
\begin{array}{lll}
                & \min\limits_{\lambda\ge 0} \lambda \\
    \text{s.t.} &  -\lambda \le \c_i^\top \x -\Delta_i(\x) \le \lambda ,& 
     \forall \x \in \hat X,\; i\ne \ell \;,\\
                &   L \le c_{i,k} \le U ,\; & i\ne \ell, 
     k=0,\ldots,m \;.
\end{array}
\end{equation}
In the above formulation of PAC model learning, the problem boils down to a linear programming (LP) optimisation. We reuse $\lambda$ to denote the optimal solution, and $\widetilde\Delta_i$ to be the function whose coefficients $\c_i$ are instantiated according to the optimal solution $\lambda$. 
Specifically, we aim to compute a PAC model $\widetilde{\bm{\Delta}}$ of $\bm{\Delta}$.
By Theorem~\ref{scenariotheorem}, the confidence and the error rate can be ensured by a sufficiently large number of samples.
Namely, to make \eqref{eq:abstConf} hold with confidence $1-\eta$,
we can choose any 
$K\ge \frac{2}{\epsilon}(\ln \frac{1}{\eta}+(m+1)(n-1)+1)$
corresponding to the number of the variables in \eqref{eq:linearP}.

For fixed $\eta$ and $\epsilon$, the number of samples $K$ is in $O(mn)$, so the LP 
problem~\eqref{eq:linearP} contains $O(mn)$ variables and $O(mn^2)$  constraints. Therefore, the computational cost of the above LP-based approach can quickly become prohibitive with increasing the dimension of input and output. 
\begin{example} \label{example:MNISTLP}
For the MNIST dataset there is the input dimension $m=28$$\times$$28=784$ and output dimension  $n=10$. Even for $\eta=0.001,\epsilon=0.4$, 
we need to solve an LP problem with $7,065$ variables and more than $630,000$ constraints, which takes up too much space (memory out with $10$GB memory). 
\end{example}
To further make the PAC model learning scale better with high-dimensional input and output,
we will consider several optimisations to reduce the complexity of the LP problem in Section~\ref{sec:efficiency}.

From the LP formulation in Eq.~\eqref{eq:linearP}, it can be seen that the PAC model learning is based on the sampling set $\hat X$ instead of the norm ball $B(\hat{\bm{x}},r)$. That is, though in this paper, for simplicity, $B(\hat{\bm{x}},r)$ is assumed to be an $L^{\infty}$-norm ball, our method also works with $L^p$-norm robustness with $1 \le p < \infty$.

\subsection{Analysing the PAC Model} 
\label{sec:generation}
We just detailed how to synthesise a PAC model  $\widetilde{\bm{\Delta}}$ of the score difference function $\bm{\Delta}$. 
When the optimisation problem in \eqref{eq:linearP} is solved, 
we obtain the PAC model $\widetilde{\bm{\Delta}}(x)\approx_{\eta,\epsilon,\lambda}\bm{\Delta}(x)$ of the score difference function.
Namely, $\widetilde{\bm{\Delta}}(\x)\pm \lambda$ approximates the upper/lower bound of the score difference function $\bm{\Delta}$ with the PAC guarantee respectively. 
As aforementioned, all maximal values of $\widetilde{\bm{\Delta}}+\lambda$
being negative implies the PAC-model robustness of the original DNN.
According to the monotonicity of affine functions,
it is not hard to compute the maximum point $\breve{\x}^{(i)}$ of $\widetilde{\Delta}_i(\x)$
in the region $B(\hat \x, r)$.
Specifically, for $\widetilde\Delta_i(\x)$ in the form of
$c_0+\sum_{j=1}^m c_jx_j$, we can infer its maximum point  directly as
\[
\breve{\x}^{(i)}_j=\left\{
\begin{aligned}
    &{\hat\x}_j+r,\qquad & {{c}}_j>0,\\
    &{\hat\x}_j-r,\qquad & {{c}}_j\leq 0.
\end{aligned}
\right.
\]
Note that the choice of $\breve{\x}^{(i)}_j$ is arbitrary for the case $c_j=0$.
Here, we choose ${\hat\x}_j-r$ as an instance.
Then let $\breve{\x}$ be the $\breve{\x}^{(i)}$
corresponding to the maximum $\widetilde{{\Delta}}_i(\breve{\x}^{(i)})$,
and the PAC-model robustness of the original DNN immediately follows
if $\widetilde\Delta(\breve{\x}) + \lambda < \bm 0$.
Besides, each $\breve{\x}^{(i)}$ is a potential adversarial example attacking the
original DNN with the classification label $i$, 
which can be further validated by checking the sign of $\Delta_i(\breve{\x}^{(i)})$.

\begin{example}
\commentout{
Now we illustrate the procedure using the fully connected neural network (FNN) as in Fig.~\ref{fig:smallnet}, where there are three layers and each layer has two neurons. The neuron activation value is calculated as the weighted sum of previous layer's neuron activations plus a bias, followed by a ReLU activation function in the hidden layer. The weight and bias parameters are highlighted on the edges and nodes respectively.
It is easy to see that the neural 
network characterises a function 
  $\bm{f}:\mathbb{R}^2\rightarrow\mathbb{R}^2$.
  For an input $\bm{x}=(x_1,x_2)^{\top}\in[-1,1]^2$, we have $\bm{f}(\x)=(f_1(\x),f_2(\x))^{\top}$.
\begin{figure}[h]
    \centering
    \begin{tikzpicture}[shorten >=1pt,->,draw=black, node distance=\layersep]
        \tikzstyle{every pin edge}=[<-,shorten <=1pt]
        \tikzstyle{neuron}=[circle,fill=black,scale=0.6]
        \tikzstyle{input neuron}=[neuron]
        \tikzstyle{output neuron}=[neuron]
        \tikzstyle{hidden neuron}=[neuron]
        \tikzstyle{annot} = [text width=4em, text centered, font=\fontsize{8}{6}]
    
        \foreach \name / \y in {1,...,2}
            \node[input neuron, pin=left:Input \y] (I-\name) at (0,-\y*1.5) {};
    
        \foreach \name / \y in {1,...,2}
            \path[yshift=0cm]
                node[hidden neuron] (H-\name) at (\layersep,-\y*1.5) {};
    
        \node[output neuron,pin={[pin edge={->}]right:Output 1}] (O1) at (\layersep*2,-1.5) {};
        \node[output neuron,pin={[pin edge={->}]right:Output 2}] (O2) at (\layersep*2,-3) {};
    
        \path (I-1) edge[above left,] node{$3$} (H-1);
        \path (I-1) edge[above left,sloped] node{$5$} (H-2);
        
        \path (I-2) edge[below left,sloped] node{$-10$} (H-1);
        \path (I-2) edge[below left] node{$-4$} (H-2);

        \path (H-1) edge[above left,sloped] node{$3$} (O1);
        \path (H-2) edge[below left,sloped] node{$1$} (O1);
        
        \path (H-1) edge[above left,sloped] node{$9$} (O2);
        \path (H-2) edge[below left,sloped] node{$7$} (O2);
    
        \node[annot,above of=I-1,color=black,node distance=0.5cm]{$[-1,1]$};
        \node[annot,above of=H-1,color=black,node distance=0.5cm] (H-1n){$-9$};
        \node[annot,above of=O1,color=black,node distance=0.5cm] (On){$14$};
        
        \node[annot,below of=I-2,color=black,node distance=0.5cm]{$[-1,1]$};
        \node[annot,below of=H-2,color=black,node distance=0.5cm] (H-2n){$-10$};
        \node[annot,below of=O2,color=black,node distance=0.5cm] (On){$-10$};
        
    \end{tikzpicture}
    \caption{An FNN with two input neurons, two hidden neurons and two output neurons.} \label{fig:ex_fnn}
    \label{fig:smallnet}
\end{figure}
}
We consider the neural network in Fig. \ref{fig:smallnet}. Given an input $\hat{\bm x}=(0,0)^{\top}$, the classification label is  $C_{\bm{f}}(\hat{\x})=1$. 
  The network is robust if $f_2(\x) < f_1(\x)$ for $\x\in B(\hat{\bm{x}},1)$, or equivalently, $f_2(\x)-f_1(\x)<0$.
Thus, our goal is to apply
the scenario approach to learn the score difference $\Delta(\x)=f_2(\x)-f_1(\x)$.
In this example, we take the approximating function of the form $\widetilde \Delta(\x)=c_0+c_1 x_1+c_2 x_2$ with constant parameters $c_0,c_1,c_2\in [-100,100]$ to be synthesised. For ease of exposition, we denote $\bm{c}=(c_1,c_2,c_3)^{\top}$.

We attempt to approximate $\Delta(\bm{x})$ by minimising the absolute difference between it and the approximating function $\widetilde{\Delta}(\bm{x})$. This process can be characterised as an optimisation problem:
\begin{equation}
    \label{optlearn}
    \begin{array}{lll}
        &\min\limits_{\bm{c},\lambda}\lambda&\\
       \mathrm{s.t.} \;& |\widetilde\Delta(\boldsymbol{x})-\Delta(\boldsymbol{x})|\leq \lambda,\;& \forall \boldsymbol{x} \in [-1,1]^2\;,\\
        &\bm{c}\in [-100,100]^3,&\\
        &\lambda\in [-100,100]\;.&
    \end{array}
\end{equation}
To apply the scenario approach, we first need to extract a set of $K$ independent and identically distributed samples $\hat{X} \subseteq [-1,1]^2$, and then reduce the optimisation problem \eqref{optlearn} to the linear programming problem
by replacing the quantifier $\forall \boldsymbol{x} \in [-1,1]^2$ with $\forall \bm{x} \in \hat{X}$ in the constraints. Theorem~\ref{scenariotheorem} indicates that at least $\lceil \frac{2}{\epsilon}(\ln \frac{1}{\eta}+4) \rceil$ samples are required to guarantee the error rate within $\epsilon$, i.e. $\mathbb{P}( |\widetilde\Delta(\bm{x})-\Delta(\bm{x})| \le \lambda )\ge 1-\epsilon$, with confidence $1-\eta$.  

Taking the error rate $\epsilon=0.01$ and the confidence $1-\eta=99.9\%$, we need (at least) $K=2182$ samples in $[-1,1]^2$. By solving the resulting linear program again, we obtain $c_0=-22.4051$, $c_1= 2.800$, $c_2= -9.095$, and $\lambda=9.821$. 

\begin{figure}[t]
    \centering
        \centering
        \includegraphics[width=0.6\linewidth]{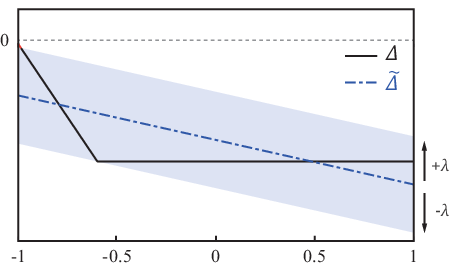}
    \caption{The functions $\Delta$ and $\widetilde\Delta$ in $x_2$ are depicted by fixing $x_1=1$. It is marked red where $\Delta(\bm{x})$ is not bounded by $\widetilde\Delta(\bm{x}) \pm \lambda$.}
    \label{fig:netlearn}
\end{figure}

For illustration, we restrict $x_1=1$, and depict the functions $\Delta$ and $\widetilde\Delta$ in Fig.~\ref{fig:netlearn}.
Our goal is to verify that the first output is always larger than the second, i.e., $\Delta(\bm{x})=f_2(\bm{x})-f_1(\bm{x})<0$. 
As described above, according to the signs of the coefficients of $\widetilde\Delta$, we obtain that $\widetilde\Delta(\bm{x})$ attains the maximum value at $\bm{x}=(1,-1)^\top$ in $[-1,1]^2$. Therefore, the network is PAC-model robustness.

\end{example}

\subsection{Strategies for Practical Analysis}\label{sec:efficiency}
We regard efficiency and scalability as the key factor for achieving  practical analysis of DNN robustness. In the following, we propose three practical  PAC-model robustness analysis techniques.

\subsubsection{Component-based learning} \label{subsec:component_learn}
As stated in Section \ref{subsec:learn_original}, the complexity of solving \eqref{eq:linearP} can still be  high, so we propose component-based learning to reduce the complexity. As before, we use $\widetilde{\Delta}_i$ to approximate $\Delta_i(\x) = f_i(\x)-f_\ell(\x)$ for each $i$ with the same template.  The idea is to learn the functions $\Delta_1,\ldots,\Delta_n$ separately,  and then combine the solutions together. Instead of solving a single large LP problem, we deal with $(n-1)$ individual smaller LP problems, each with $O(m)$ linear constraints. 
 As a result, we have $\widetilde\Delta_i(\x)\approx_{\eta,\epsilon,\lambda_i}\Delta_i(\x)$, from which we can only deduce that
\[
\mathbb{P} \Big(\bigwedge_{i\ne \ell}|\widetilde\Delta_i(\x)-\Delta_i(\x)|\le \lambda_i \Big)\ge 1-(n-1)\epsilon 
\]
with the confidence decreasing to at most  $1-(n-1)\eta$. To guarantee the error rate at least $\epsilon$ and the confidence at least $1-\eta$, we need to recompute the error $\lambda$ between $\widetilde{\bm{\Delta}}(\x)$ and $\bm{\Delta}(\x)$. Specifically, we solve the following optimisation problem constructed by resampling:

\begin{equation}
    \label{dimensionlearn}
    \begin{array}{lll}
                & \min\limits_{\lambda} \lambda & \\
        \text{s.t.} & |\widetilde\Delta_i(\bm{x})-\Delta_i(\bm{x})|\le \lambda ,\;\\
        & \forall \bm{x}\in \hat X\;,i\ne \ell.
    \end{array}
\end{equation}
where $\hat X$ is a set of $K$ i.i.d samples with $K\ge \tfrac{2}{\epsilon}(\ln \tfrac{1}{\eta}+1)$.
Applying  Theorem~\ref{scenariotheorem} again, we have $\widetilde{\bm{\Delta}}(\x)\approx_{\eta,\epsilon,\lambda}\bm{\Delta}(\x)$ as desired.

We have already relaxed the optimisation problem \eqref{eq:linearP} into a family of $(n-1)$ small-scale LP problems. If $n$ is too large (e.g. for Imagenet with $1000$ classes), we can also consider the untargeted score difference function
 $ {\Delta}_{\mathrm u} (\bm x) = \bm f _\ell (\bm x) - \max_{i \ne l} \bm f_i(\bm x)$.
By adopting the untargeted score difference function, the number of the LP problems is reduced to one.
The untargeted score difference function improves the efficiency at expense of the loss of linearity, which harms the accuracy of the affine model.

\commentout{
\begin{example}
Let us reconsider Example~\ref{example:MNISTLP} in Sect.~\ref{subsec:learn_original}. Setting parameters $\eta=0.001,\epsilon=0.1$, a resulting individual LP problem can be solved in $89.4$ seconds with less than $8$GB memory; however, for a smaller error rate such as $\epsilon=0.02$, it takes more than $2$ hours and then raises a solver error. Thus we have to reduce the size of the optimisation problem further.
\end{example}
}

\subsubsection{Focused learning}
\label{subsec:focus_learning}
In this part, our goal is to reduce the complexity further by dividing the learning procedure into two phases with different fineness: i)
 in the first phase, we use a small 
 set of samples to extract coefficients with big absolute values; and ii) these coefficients are ``focused'' in the second phase, in which we use more samples to refine them. In this way, we reduce the number of variables overall, and we call it \emph{focused learning}, which namely refers to focusing the model learning procedure on important features. It is embedded in the component learning procedure.

The main idea of focused learning is depicted below:
\begin{enumerate}
\item \emph{First learning phase:} We extract $K^{(1)}$ i.i.d. samples from the input region $B(\hat{\bm{x}},r)$. We first learn $\Delta_i$ on the $K^{(1)}$ samples. Thus, our LP problems have $O(K^{(1)})$ constraints with $O(m)$ variables. For large datasets like 
ImageNet, the resulting LP problem is still too large. We use efficient learning algorithms such as linear regression (ordinary least squares) to boost the first learning phase on these large datasets.
\item \emph{Key feature extraction:} After solving the LP problem  (or the linear regression for large datasets), we synthesise $\widetilde\Delta_i^{(1)}$ as the approximating function.
Let $\mathit{KeyF}_i(\kappa)\subseteq \{1,x_1,\ldots,x_m\}$ denote the set of extracted key features for the $i$th component corresponding to the $\kappa$ coefficients with the largest absolute values in $\widetilde\Delta_i^{(1)}$. 
\item \emph{Focused learning phase:} We extract $K^{(2)}$ i.i.d. samples from $B(\hat{\bm{x}},r)$. For these samples, we generate constraints only for our key features in $\mathit{KeyF}_i(\kappa)$ by fixing the other coefficients using those in $\widetilde\Delta_i^{(1)}$, and thus the number of the undetermined coefficients is bounded by $\kappa$. By solving an LP problem comprised of these constraints, we finally determine the coefficients of the features in $\mathit{KeyF}_i(\kappa)$.
\end{enumerate}

We can determine the sample size $K^{(2)}$ and the number of key features $\kappa$ satisfying
\[
    \kappa \leq \frac{K^{(2)} \epsilon}{2}-\ln\frac{1}{\eta}-1\;,
\]
which can be easily inferred from Theorem \ref{scenariotheorem}.
It is worth mentioning that, focused learning not only significantly improves the efficiency, but it also makes our approach insensitive to significance level $\eta$ and error rate $\epsilon$, because the first phase in focused learning can provide a highly precise model, and a small number of samples are sufficient to learn the PAC model in the second phase. This will be validated in our experiments.

\subsubsection{Stepwise splitting}
\label{sec:stepwise-splitting}

\begin{figure}[t]
    \centering
        \centering
        \includegraphics[width=0.75\linewidth]{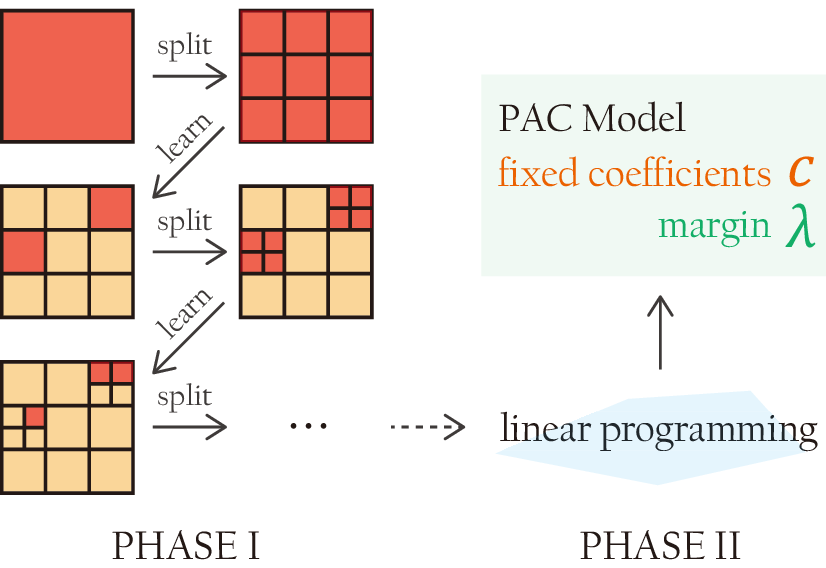}
    \caption{A workflow of the stepwise splitting procedure. The red color indicates the significant grids whose coefficients will be further refined,
    while the yellow color indicates the grids whose coefficients have been determined. }
    \label{fig:stepwisesplit}
\end{figure}

When the dimensionality of the input space is very high (e.g., ImageNet), 
The first learning phase of focused learning requires constraints generated
by tons of samples to make precise predictions on the key features,
which is very hard and even impossible to be directly solved. 
For achieving better scalability, 
we partition the dimensions of input $\{1,\ldots,m\}$ into groups $\{G_k\}$.
In an affine model $\widetilde{\Delta}_i$, 
for the variables with undetermined coefficients 
in each certain group $G_k$, they share the same coefficient $c_k$. 
Namely, the affine model has the form of $\sum_k \left(c_k \sum_{i\in G_k} x_i\right)$.
Then, a coarse model can be learned.

We compose the refinement into the procedure of focused learning aforementioned (See Fig.~\ref{fig:stepwisesplit}).
Specifically, after a coarse model is learned, we fix the coefficients for the insignificant groups and extract the key groups. The key groups are then further refined, and their coefficients are renewed by learning on a new batch of samples. We repeat this procedure iteratively until
 most coefficients of the affine model are fixed, and then we invoke linear programming to compute the rest coefficients and the margin. This iterative refinement can be regarded as multi-stage focused learning with different fineness.

In particular, for a colour image, we can use the grid to divide its pixels into groups. The image has three channels corresponding to the red, green and blue levels. As a result, each grid will generate 3 groups matching these channels, i.e. $G_{k,\mathrm{R}}$, $G_{k,\mathrm{G}}$, and $G_{k,\mathrm{B}}$. Here, we determine the significance of a grid with the $L^2$-norm of the coefficients of its groups, i.e. $(c_{k,\mathrm{R}}^2+c_{k,\mathrm{G}}^2+c_{k,\mathrm{B}}^2)^{\frac{1}{2}}$. 
Then the key groups (saying corresponding to the top $25\%$ significant grids) will be further refined in the subsequent procedure.
On ImageNet, we initially divide the image into $32\times32$ grids, with each grid of the size $7\times7$. In each refinement iteration, we split each significant grid into 4 sub-grids (see Fig.~\ref{fig:stepwisesplit}). We perform 6 iterations of such refinement and use $20\,000$ samples in each iteration. An example on stepwise splitting of an ImageNet image can be found in Fig.~\ref{fig:expsplit} in Sect.~\ref{subsec:exp3}.

%% file: related.tex
\section{Related Work}
\label{sec:related_work}

Here we discuss more results on the verification, adversarial attacks and testing for DNNs. A number of formal verification techniques have been proposed for DNNs, including constraint-solving~\cite{reluplex,planet,DBLP:conf/aaai/NarodytskaKRSW18,DBLP:journals/jmlr/BunelLTTKK20,DBLP:conf/cvpr/LinYCZLLH19,Simplify2020,Deepsafe,wangji}, abstract interpretation~\cite{AI2,deepsymbol,deepz,deeppoly,deepsrgrextended}, layer-by-layer exhaustive search~\cite{DBLP:conf/cav/HuangKWW17}, global optimisation~\cite{RHK2018,DBLP:conf/nfm/DuttaJST18,DBLP:conf/ijcai/RuanWSHKK19}, convex relaxation~\cite{AircraftSafeLinear,ReluDiff,NeuroDiff}, functional approximation~\cite{fastlin}, reduction to two-player games~\cite{DBLP:conf/tacas/WickerHK18,DBLP:journals/tcs/WuWRHK20}, and star-set-based abstraction~\cite{starset19,imagestar20}. 
Sampling-based methods are adopted to probabilistic robustness verification in \cite{statistical19,DBLP:conf/icml/WengCNSBOD19,Probalistic2019,BayesianAAAI,anderson2020certifying,anderson2020datadriven}. Most of them provide sound DNN robustness estimation in the form of a norm ball, but typically for very small networks or with pessimistic estimation of the norm ball radius.
By contrast, statistical methods \cite{statistical19,DBLP:conf/icml/WengCNSBOD19,baluta2021scalable,quanti4,quanti5,quanti6,BayesianIJCAI, huangpei01} are more efficient and scalable when the structure of DNNs is complex. The primary difference between these methods and \deeppac is that our method is model-based and thus more accurate. 
We use samples to learn a relatively simple model of the DNN with the PAC guarantee via scenario optimisation and gain more insights to the analysis of adversarial  robustness.
The generation of adversarial inputs \cite{SZSBEGF2014} itself has been widely studied by a rich literature of adversarial attack methods. 
Some most well-known robustness attack methods include Fast Gradient Sign \cite{fgsm}, Jacobian-based saliency map approach \cite{DBLP:conf/eurosp/PapernotMJFCS16}, C\&W attack \cite{CW-Attacks}, etc. 
Though adversarial attack methods generate adversarial inputs efficiently, they cannot enforce guarantee of any form for the DNN robustness. Testing is still the primary approach for certifying the use of software products and services. In recent years, significant work has been done for the testing for DNNs such as test coverage criteria specialised for DNNS \cite{pei2017deepxplore,ma2018deepgauge,sun2019structural,kim2019guiding,yan2020correlations} and different testing techniques adopted for DNNs \cite{tian2018deeptest,zhang2018deeproad,xie2019deephunter,ma2018deepmutation,sun2020automatic,wang2021robot,humbatova2021deepcrime,riccio2020model,ma2018mode,yan2021exposing}. In particular, our experiments show that the results from \deeppac are consistent with the DNN testing work for prioritising test inputs \cite{feng2020deepgini,wang2021prioritizing}, but with a stronger guarantee. This highlights again that \deeppac is a practical verification method for DNN robustness.